\documentclass{article}

\usepackage[utf8]{inputenc}
\usepackage{amssymb, amsmath, amsthm}
\usepackage{algorithm, algpseudocode}
\usepackage{graphicx}
\usepackage{booktabs}
\usepackage{multirow,multicol}
\usepackage{hyperref}

\newtheorem{lemma}{Lemma}[section]

\newtheorem{theorem}[lemma]{Theorem}
\newtheorem{corollary}[lemma]{Corollary}

\DeclareMathOperator{\supp}{supp}
\newcommand{\real}{\mathbb{R}}
\newcommand{\dualp}[1]{\left\langle #1 \right\rangle} 

\newcommand{\E}[1]{\mathbb{E}\left[ #1 \right]}
\newcommand{\pr}[1]{\operatorname{Pr}\left[ #1 \right]}
\newcommand{\bv}{\boldsymbol{v}}
\newcommand{\lmax}{\lambda_{\operatorname{max}}}
\DeclareMathOperator{\trace}{tr}

\newcommand{\dom}{D}
\newcommand{\wdom}{\Theta}
\newcommand{\hs}{\mathcal{H}}

\newcommand{\ptheta}{{\bar{\theta}}}
\newcommand{\res}{\kappa}
\newcommand{\wdiff}{h}
\newcommand{\constres}{\gamma}
\newcommand{\Constres}{\Gamma}
\newcommand{\indexset}{\mathcal{I}}
\newcommand{\bgamma}{\bar{\gamma}}

\begin{document}

\title{Approximation results for Gradient Descent trained Shallow Neural Networks in $1d$}

\author{R. Gentile${}^*$, G. Welper${}^{*,\dagger}$}
\renewcommand{\thefootnote}{$*$}
\footnotetext{Department of Mathematics, University of Central Florida, Orlando, FL 32816, USA}
\renewcommand{\thefootnote}{$\dagger$}
\footnotetext{Corresponding author: \href{mailto:gerrit.welper@ucf.edu}{\texttt{gerrit.welper@ucf.edu}}}
\footnotetext{This material is based upon work supported by the National Science Foundation under Grant No. 1912703.}
\renewcommand{\thefootnote}{\arabic{footnote}}

\date{}
\maketitle

\begin{abstract}

  Two aspects of neural networks that have been extensively studied in the recent literature are their function approximation properties and their training by gradient descent methods. The approximation problem seeks accurate approximations with a minimal number of weights. In most of the current literature these weights are fully or partially hand-crafted, showing the capabilities of neural networks but not necessarily their practical performance. In contrast, optimization theory for neural networks heavily relies on an abundance of weights in over-parametrized regimes. 
  
  This paper balances these two demands and provides an approximation result for shallow networks in $1d$ with non-convex weight optimization by gradient descent. We consider finite width networks and infinite sample limits, which is the typical setup in approximation theory. Technically, this problem is not over-parametrized, however, some form of redundancy reappears as a loss in approximation rate compared to best possible rates.
  
\end{abstract}



\section{Introduction}

A large amount of recent work demonstrates approximation results for neural networks. They typically provide inequalities
\begin{equation} \label{eq:approx}
  \begin{aligned}
    \inf_\theta \|f_\theta - g\| & \le n(\theta)^{-r}, &
    g & \in K
  \end{aligned}
\end{equation}
that bound the approximation error $\|f_\theta - g\|$ between a neural network $f_\theta$ and a function $f$ in a compact set $K$ by some small bound depending on size $n(\theta)$ of the network like width, depth or total number of weights $\theta$. In the vast majority of papers, the weights are handcrafted, so that they demonstrate what neural networks can potentially do, without any guarantee that this is realized by practical training algorithms.

In contrast to the approximation problem, a significant part of the current literature on neural network optimization is based on linearization and the \emph{neural tangent kernel (NTK)} and works in an \emph{over-parametrized} regime:  The error is replaced by a discrete least squares loss
\[
  \|f_\theta - g\|_* := \left( \frac{1}{n} \sum_{i=1}^n |f_\theta(x_i) - g(x_i)|^2 \right)^{1/2},
\]
or similar sample losses, and we consider neural networks with enough capacity to achieve zero error
\[
  \inf_\theta \|f_\theta - g\|_* = 0.
\]
In fact, contemporary global training proofs require a significant safety margin of redundant weights. This is diametrically opposed to approximation theory, which aims for as few degrees of freedom as possible to represent $g$. In addition, for practically relevant compact sets $K$, zero approximation error cannot be expected so the approximation in question is necessarily in an \emph{under-parametrized} regime.

This paper seeks a combination of the approximation and optimization perspectives and provides an approximation error bound of type \eqref{eq:approx} with weights that are trained by gradient flow and therefore practically achievable. The convergence rates are sub-optimal so that for any achieved error the networks have some redundant capacity. This replaces the over-parametrization in the optimization literature, which measure the redundancy in relation to the number of samples and is thus not applicable in infinite sample limit of the function space norm $\|\cdot\|$. Since neural network approximation results with training guarantees are currently poorly understood, we only consider shallow networks in one dimension, with the main result stated in Theorem \ref{th:convergence-bias1d-short}. The proof contains an abstracted variant, Theorem \ref{th:convergence-general}, which seems promising for multiple dimensions and deeper networks.

The paper is organized as follows. Section \ref{sec:main} contains the training setup and the main results, including a short overview of the proof. Section \ref{sec:experiments} contains some numerical experiments. Section \ref{sec:proof-abstract} contains the proof of the abstracted main result Theorem \ref{th:convergence-general} and Section \ref{sec:proof-1d} contains the proof for $1d$ networks Theorem \ref{th:convergence-bias1d-short}. Finally Section \ref{sec:auxiliary} contains a review of matrix concentration inequalities and some technical lemmas.

\paragraph{Literature Review}

\begin{itemize}

  \item \emph{Approximation:} Approximation results for neural networks typically establish error bounds of type \eqref{eq:approx} for varying network architectures and function classes $K$. With few exceptions, the weights of the networks are hand-crafted, so that these results prove the theoretical capability of the networks but not their practical performance. Some recent overviews are given in
  \cite{
    Pinkus1999,
    DeVoreHaninPetrova2020,
    WeinanChaoLeiWojtowytsch2020,
    BernerGrohsKutyniokPetersen2021%
  }. Results for classical function classes $K$ like Sobolev and Besov spaces can be found in
  \cite{
    GribonvalKutyniokNielsenEtAl2019,
    GuhringKutyniokPetersen2020,
    OpschoorPetersenSchwab2020,
    LiTangYu2019,
    Suzuki2019%
  }.
  Some papers
  \cite{
    Yarotsky2017,
    Yarotsky2018,
    YarotskyZhevnerchuk2020,
    DaubechiesDeVoreFoucartEtAl2019,
    ShenYangZhang2019,
    LuShenYangZhang2021%
  }
  report better approximation rates than classical methods, for the price of discontinuous weight assignments. Another group of papers classifies approximation orders in terms of Barron spaces
  \cite{
    Bach2017,
    KlusowskiBarron2018,
    WeinanMaWu2019,
    LiMaWu2020,
    SiegelXu2020,
    SiegelXu2020a,
    BreslerNagaraj2020%
  }.
  Many papers address specialized function classes
  \cite{
    ShahamCloningerCoifman2018,
    PoggioMhaskarRosascoEtAl2017%
  },
  often from applications like PDEs
  \cite{
    KutyniokPetersenRaslanSchneider2022,
    PetersenVoigtlaender2018,
    LaakmannPetersen2020,
    MarcatiOpschoorPetersenSchwab2022%
  }.
  Many of the papers above contain lower approximation bounds that establish limits of neural network approximation, as well. Some extra results in this direction can be found in
  \cite{
    ElbraechterPerekrestenkoGrohsEtAl2019%
  }.

  \item \emph{Optimization:} Due to the vast literature on neural network optimization, the following review is confined to the neural tangent kernel approach (NTK) used in this paper. It was introduced in
  \cite{
    JacotGabrielHongler2018%
  }
  and
  \cite{
    LiLiang2018,
    Allen-ZhuLiSong2019,
    DuZhaiPoczosSingh2019,
    DuLeeLiEtAl2019%
  },
  which obtain global gradient decent convergence in over-parametrized regimes by a perturbation analysis, see also
  \cite{
    ZouCaoZhouGu2020,
    AroraDuHuEtAl2019a,
    LeeXiaoSchoenholzBahriNovakSohlDicksteinPennington2019,
    SongYang2019,
    ZouGu2019,
    KawaguchiHuang2019,
    ChizatOyallonBach2019,
    OymakSoltanolkotabi2020,
    NguyenMondelli2020,
    BaiLee2020,
    SongRamezaniKebryaPethickEftekhariCevher2021,
    LeeChoiRyuNo2022%
  }.
  Some papers 
  \cite{
    AroraDuHuEtAl2019,
    SuYang2019,
    JiTelgarsky2020,
    ChenCaoZouGu2021%
  } 
  provide a refined analysis with over-parametrization or convergence speed depending on properties of the target function $g$, usually in terms of an expansion in the NTK eigenbasis. These properties are closely related to the smoothness requirements in this paper, but unlike the references, we do not require an over-parametrized regime.

  Since the NTK argument relies on linearization, a large number of papers asks to what extent it accurately describes nonlinear neural network training
  \cite{
    GeigerJacotSpiglerGabrielSagundAscoliBiroliHonglerWyart2019,
    HaninNica2020,
    FortDziugaitePaulKharaghaniRoyGanguli2020,
    LeeSchoenholzPenningtonAdlamXiaoNovakSohl-Dickstein2020,
    SeleznovaKutyniok2022a%
  }
  Characterizations of the NTK's reproducing kernel Hilbert space are provided in
  \cite{
    BiettiMairal2019,
    GeifmanYadavKastenGalunJacobsRonen2020,
    ChenXu2021%
  } and approximation properties in
  \cite{
    JiTelgarskyXian2020%
  }.
  Convergence analysis for optimizing linearized NTK models can be found in
  \cite{
    VelikanovYarotsky2021,
    VelikanovYarotsky2022%
  }.

  \item \emph{Approximation and Optimization:}
  Although the literature on neural network approximation on the one side and neural network optimization on the other side is quite extensive, there is little work on the combination of the two areas to provide trainable approximation results. The papers
  \cite{
    AdcockDexter2020,
    GrohsVoigtlaender2021%
  }
  study the gap between theory and practice for a finite amount of samples. Approximation results with training guarantees are given in
  \cite{
    SiegelXu2022,
    HaoJinSiegelXu2021%
  }
  based on greedy algorithms instead of gradient descent and in
  \cite{
    HerrmannOpschoorSchwab2022%
  }
  based on a two step procedure involving a classical and subsequent neural network approximation.

\end{itemize}

\section{Problem Setup and Main Result}
\label{sec:main}

\subsection{Notations}

Throughout the paper, $c$ and $C$ denote generic constants that may change in every occurrence. Likewise $\lesssim$, $\gtrsim$ and $\sim$ denote less, greater and equivalent up to a generic constant. In all cases these constants are independent of the network width $m$ and the function $g$ we want to approximate.

One exception are the two constants $\constres$ and $\Constres$, which are only used in the shortened variants of the main result, Theorem \ref{th:convergence-bias1d-short} and Corollaries \ref{cor:approximation-error-bias1d-short}, \ref{cor:approximation-error-initial}. They are allowed to depend the initial training loss as stated in the respective results.

\subsection{Shallow Neural Network}
\label{sec:shallow}

\paragraph{The Neural Network}

We approximate a function $g(x)$ on the interval $[-1, 1]$ with the shallow neural network
\begin{equation} \label{eq:nn}
  f_\theta(x) := \frac{1}{\sqrt{m}} \sum_{r=1}^m a_r \sigma(x - \theta_r)
\end{equation}
of width $m$, with trainable biases $\theta = [\theta_1, \dots, \theta_m]^T \in \real^m$, untrained random weights $a = [a_1, \dots, a_m]^T \in \{-1,+1\}^m$ and no multiplicative weights in the first layer. This problem is comparable to the network
\[
  \frac{1}{\sqrt{m}} \sum_{r=1}^m a_r \sigma(w_r x),
\]
which is well studied in the literature \cite{LiLiang2018,DuZhaiPoczosSingh2019,AroraDuHuEtAl2019,OymakSoltanolkotabi2020}, again with trained $w_r$ and untrained $a_r$, but multi-dimensional $x$ and the extra condition $|x| = 1$. Of course, the latter condition is too restrictive for $1d$ approximation, however, without it the network output is $0$ for $x=0$ no matter the choice of the weights $w_r$. Therefore, we cannot obtain a universal function approximation result on our domain $x \in [-1, 1]$ and resort to the optimization of the biases instead of the weight matrices. As we will see, that restores the universal function approximation property. The network \eqref{eq:nn} is chosen as a model for neural network training that is as simple as possible, while retaining a non-convex optimization problem.

Since the neural network is a linear combination of the functions $\sigma(x-\theta_r)$, it seems peculiar not to optimize the coefficients $a_r$ as well. However, even if we trained all $a_r$, in the linearized neural tangent kernel regime they do not move far away from their initial value and therefore do not achieve the values one would expect from basis coefficients. See the numerical experiments in Figure \ref{fig:weight-distance} and some analogous bounds for the biases in Corollary \ref{cor:approximation-error-initial}. On the other hand, the proof would be significantly more technical. Therefore, we freeze the $a_r$ and leave a discussion of more complete and deeper networks to some comments after Theorem \ref{th:convergence-general} and future research.

\paragraph{Loss Function}

Since this paper aims for an approximation result that is achievable by gradient flow training, we optimize the $\hs := L_2([-1,1])$, loss
\[
  L(\theta) := \|f_\theta - g\|_{L_2}^2 = \int_{-1}^1 |f_\theta(x) - g(x)|^2 \, dx.
\]
This loss can be understood as the infinite sample limit, or the generalization error, of the standard least squares loss. Since our network has finite width, the training problem is \emph{under-parametrized}. This is a significant difference from the current literature, which considers problems with (significantly) more width than samples and is therefore \emph{over-parametrized}.

\paragraph{Gradient Flow}

In order to keep the exposition simple, we consider the optimization of the loss by gradient flow
\begin{equation} \label{eq:shallow:gradient-flow}
  \frac{d}{dt} \theta(t) = - \nabla L(\theta).
\end{equation}
The biases $\theta_r(0)$ are initialized by uniform sampling on the interval $[-1, 1]$. The weights $a_r$ are $\pm 1$, each with probability $1/2$ and are not trained.

\paragraph{Smoothness}

In order to obtain approximation results, we must assume that $g$ is contained in a suitable compact set $K$. Classically, these are balls in Sobolev or Besov spaces. We construct a smoothness norm based on the neural tangent kernel, which will turn out to be similar to the Sobolev space $H^s(\dom)$. To this end, we consider the neural network as a function $\theta \to f_\theta(\cdot)$, mapping the weights to the Hilbert space $\hs = L_2([-1,1])$ and we define the \emph{neural tangent kernel} by
\begin{align*}
  & H: \hs \to \hs, &
  H & := \E{D f_\theta (Df_\theta)^*},
\end{align*}
where the expectation is taken over random initial values $\theta(0)$ and the derivative $D$ with respect to the weights $\theta$. This operator is compact and self-adjoint and therefore has an orthonormal basis of eigenfunctions $\phi_k$ with eigenvalues $2 \omega_k^{-2}$. We expand $v \in \hs$ in the eigenbasis $\bv = \sum_{1=k}^\infty v_k \phi_k$, collect all coefficients in the vector $\bv = [v_k]_{k=1}^\infty$ and define the smoothness norm
\begin{equation} \label{eq:shallow:s-norm}
  \|v\|_s := \|\bv\|_s := \left( \sum_{k=1}^\infty \omega_k^{2s} v_k^2 \right)^{1/2}
\end{equation}
of any function $v$ for which the coefficients $v_k$ decay sufficiently fast, together with the corresponding Hilbert space $\hs^s$. Note in particular that for $s=0$ we have $\|v\|_\hs = \|\bv\|_0$. Likewise, we define $\|v\|_s := \|\bv\|_s$ and do not distinguish between $v$ and $\bv$ in the following. For our shallow network, it is not hard to show that the eigenvalues and eigenfunctions are given by
\begin{align*}
  \omega_k & = \frac{\pi}{4} + \frac{\pi}{2} k, &
  \phi_k(x) & = \left\{ \begin{array}{ll}
    \sin\left(\omega_k x - \frac{\pi}{4} \right) & k\text{ even} \\
    \sin\left(\omega_k x + \frac{\pi}{4} \right) & k\text{ odd.}
    \end{array} \right.
\end{align*}
Therefore, the smoothness norm is closely related to the Sobolev space $H^s([-1,1])$ up to some modified boundary conditions satisfied by the eigenfunctions.

\paragraph{Main Result}

We now state the main result of this paper.

\begin{theorem}[Simplified Version of Theorem \ref{th:convergence-bias1d}] \label{th:convergence-bias1d-short}
  Let $\res = \res(t) = f_{\theta(t)} - g$ be the residual of the gradient flow \eqref{eq:shallow:gradient-flow} for the shallow network \eqref{eq:nn}. Then for the smoothness norms $\|\cdot\|_s$ defined in \eqref{eq:shallow:s-norm}, for $0 < s < \frac{1}{2}$ and $m \ge \constres$, with probability at least $1 - \Constres m^{\frac{1}{2-s}} e^{-m^{\frac{1}{2-s}}} - \Constres m^{\frac{1}{4-2s}} e^{-\gamma m^{\frac{1-s}{2-s}}}$ we have
  \[
    \|\res(t)\|_0^2
    \le \Constres \left[ m^{-\frac{1}{2} \frac{1-s}{2-s}} \|\res(0)\|_s^{\frac{2}{s}} + e^{- \constres \frac{t}{s} m^{-\frac{1}{2} \frac{1-s}{2-s}}} \right]^s
  \]
  for some generic constants $\constres,\Constres$ that depend on $s$ and $\|\res(0)\|_0$ and do not depend on $m$ and $\|\res(0)\|_s$.

\end{theorem}

Since the constants $\constres$ and $\Constres$ may depend on $\|\res(0)\|_0$, we use a different naming convention than for the $\res$ independent generic constants $c$, $C$ used in the rest of the paper. A variant of the theorem that unties $\|\res(0)\|_0$ from the constants is given in Theorem \ref{th:convergence-bias1d} below. It is based on a neural tangent kernel argument (NTK), which ensures that the weights do not move far from their initialization so that the network can be linearized around the initial weights. In the error bound of the theorem, the first summand remains constant during training, whereas the second tends to zero. Once it is smaller than the first, we obtain the following error bound of $\|f_{\theta(t)} - g\| = \|\res(t)\|$, up to a change in the constant $C$.

\begin{corollary} \label{cor:approximation-error-bias1d-short}
  With the assumptions and success probability from Theorem \ref{th:convergence-bias1d-short}, gradient flow training achieves the error
  \[
    \|f_{\theta(t)} - g\|_0
    \le \Constres m^{-\frac{s}{4} \frac{1-s}{2-s}} \|\res(0)\|_s
  \]
  for sufficiently large $t$.
\end{corollary}

For comparison, consider uniform splines $S_m$ with $m$ breakpoints. They are included in our shallow networks (if we allow training of $a_r$) and achieve the analogous error bound
\[
  \begin{aligned}
    \inf_{f \in S^u_m} \|f - g\|_0 & \le c m^{-s} \|g\|_{H^s([-1,1]}, &
    s & \le 2,
  \end{aligned}
\]
with two notable differences: First, the theorem requires the more restrictive range $s \le \frac{1}{2}$ and second, for given smoothness $s$, it has a lower approximation rate $\frac{s}{4} \frac{1-s}{2-s} < s$. These can be partially motivated as follows. Whereas the spline is a sum of shape functions $\frac{a_r}{\sqrt{m}} \sigma(\cdot - \theta_r)$ with arbitrary coefficients $\tilde{a}_r = \frac{a_r}{\sqrt{m}}$, our networks restrict the coefficients $\tilde{a}_r$ to $\pm \frac{1}{\sqrt{m}}$. However, by stacking multiple shape functions $\sigma(\cdot - \theta_r)$ with identical breakpoints $\theta_r$, we can emulate any coefficient in increments of $\frac{1}{\sqrt{m}}$. Hence a higher approximation rate cannot be expected. In addition, the stacks require more shape functions than a direct adjustment of a single $a_r$ and thus lead to a reduction in the approximation rate. Trainable $a_r$ do not fully resolve this issue because they remain close to their initialization $\pm 1$ as long as we stay in the NTK regime, which relies on the closeness for linearization. Finally, the loss in rate allows some redundancy in the networks weights, which is comparable to over-parametrization in the finite sample case.

Results competitive with adaptive splines cannot be expected from the NTK linearization argument, either. The observation that the breakpoints $\theta_r$ remain close to their uniform initialization prevents any meaningful adaption to the target function $g$, as stated in the following corollary.

\begin{corollary} \label{cor:approximation-error-initial}
  With the assumptions and success probability from Theorem \ref{th:convergence-bias1d-short}, for any $\theta(0) = [\theta_1(0), \dots, \theta_m(0)]$ uniformly random on $[-1, 1]$, there exits $\theta = [\theta_1, \dots, \theta_m]$ with
  \begin{align*}
    \|\theta - \theta(0)\|_\infty & \le \constres m^{- \frac{1}{2} \frac{1}{2-s}}
    \\
    \|f_\theta - g\|_0 & \le \Constres m^{-\frac{s}{4} \frac{1-s}{2-s}} \|\res(0)\|_s
  \end{align*}
\end{corollary}

\begin{proof}
  This directly follows from \eqref{eq:wdiff} with $\alpha = 1-s$ and $h = m^{-\frac{1}{2} \frac{1}{2-s}}$ from Theorem \ref{th:convergence-bias1d}.
\end{proof}

Note that standard neural network approximation arguments show that there exists a single, usually hand crafted, set of weights for which the neural network achieves a good error. In contrast, the corollary states that an error bound can be guaranteed by weights close to any random initialization. Intuitively, this eases optimization at the cost of approximation.

\subsection{Abstract Result}
\label{sec:result-abstract}

For the proof of Theorem \ref{th:convergence-bias1d-short}, it is convenient to work in an abstracted framework that allows a generalized convergence result, stated in this section. The neural network $f_\theta$ and space $\hs$ are relaxed to a generic Fr\'{e}chet differentiable function
\begin{align*}
  & f: \wdom = \ell_2(\real^m) \to \hs, & \theta & \to f_\theta
\end{align*}
for an arbitrary Hilbert space $\hs$. Then, we optimize the loss
\[
  L(\theta) = \|f_\theta - g\|_\hs^2
\]
for some $g \in \hs$ by gradient flow training
\begin{equation} \label{eq:gradient-flow}
  \frac{d}{dt} \theta(t) = - \nabla L(\theta),
\end{equation}
with independent weights initialized by some probability distribution $P$ on $\real$. Furthermore, for $S \in [-1,s]$, let $\hs^S$ be Banach spaces, with norms $\|\cdot\|_S$ such that $\hs^0 = \hs$ and for any $s<r<t$ the interpolation inequality
\begin{equation} \label{eq:interpolation-inequality}
    \|v\|_r \le \|v\|_s^{\frac{t-r}{t-s}} \|v\|_t^{\frac{r-s}{t-s}}
\end{equation}
holds for any $v \in \hs^t$. Then, with the neural tangent kernel
\begin{align*}
  H & := \E{D f_\theta (Df_\theta)^*},
\end{align*}
we obtain the following convergence result.

\begin{theorem} \label{th:convergence-general}
  Let $\res = \res(t) = f_{\theta(t)} - g$ be the residual. For $S \in \{0, s\}$, constants $c_\infty, c_0, \alpha > 0$ and function $p_0(m,\wdiff)$ assume that
  \begin{enumerate}

    \item \label{item:convergence-general:weight-distance} The distance of the weights from their initial value is controlled by
    \begin{equation} \label{eq:convergence-general:weight-diff}
      \|\theta(t) - \theta(0)\|_\infty
      \lesssim \sqrt{\frac{2}{m}} \int_0^t \|\res(\tau)\|_0 \, d\tau.
    \end{equation}

    \item The norms of $\hs^r$, $r \in [-1, s]$ satisfy the interpolation inequality \eqref{eq:interpolation-inequality}.

    \item \label{item:convergence-general:coercive} $H$ is coercive
    \begin{align} \label{eq:convergence-general:coercive}
      \|v\|_{S-1}^2 & \lesssim \dualp{v, H v}_S, &
      v & \in \hs^{S-1}
    \end{align}

    \item \label{item:convergence-general:initial}
    The partial derivatives $\partial_r f_\theta$, $r=1, \dots, m$ only depend on the single weight $\theta_r$ and for some $\mu \ge 0$
    \begin{equation} \label{eq:convergence-general:initial}
      \|\partial_r f_\theta\|_S \le \frac{\mu}{\sqrt{m}}.
    \end{equation}

    \item \label{item:convergence-general:perturbation}
    For any $\wdiff>0$, with probability at least $p_0(m,\wdiff)$, for all $\ptheta \in \wdom$ with $\|\ptheta - \theta(0)\|_\infty \le h$ we have
    \begin{equation} \label{eq:convergence-general:perturbation}
      \sup_{\|\nu\|_\infty \le 1} \left\| \sum_{r=1}^m (\partial_r f_\theta - \partial_r f_\ptheta) \nu_r \right\|_S \le \sqrt{m} L \wdiff^\alpha
    \end{equation}
    for some $\alpha, L \ge 0$.

  \end{enumerate}
  Then for $m$ sufficiently large so that $\sqrt{\frac{c_\infty \tau}{m}} \le 1$, with probability at least $1 - 2 \tau \left(e^\tau - \tau - 1 \right)^{-1} - p_0(m,\wdiff)$ we have
  \begin{equation} \label{eq:convergence-general}
    \|\res(t)\|_0^2
    \lesssim \left[ \wdiff^\alpha \|\res(0)\|_s^{\frac{2}{s}} + \|\res(0)\|_0^{\frac{2}{s}} e^{-\wdiff^\alpha \frac{t}{s}} \right]^s
  \end{equation}
  with
  \begin{align*}
    \wdiff & = c_\wdiff^{\frac{1}{1+\alpha}} \|\res(0)\|_0^{\frac{1}{1+\alpha}} m^{-\frac{1}{2} \frac{1}{1+\alpha}}, &
    \tau & = c_0^{\frac{2\alpha}{1+\alpha}} \|\res(0)\|_0^{\frac{2\alpha}{1+\alpha}} m^{\frac{1}{1+\alpha}}
  \end{align*}
  for some constants $c_\wdiff,c,C \ge 0$ dependent of $s$ and independent of $\res$ and $m$.

\end{theorem}

The proof is given in Section \ref{sec:proof-abstract}. Some Remarks:

\begin{itemize}
  \item The gradient flow error bound \eqref{eq:convergence-general} consists of two terms: One summand tends to zero with time and one summand of small size $\wdiff^\alpha$ that is independent of $t$ and constitutes the final error after training. This is analogous to Theorem \ref{th:convergence-bias1d-short} in the shallow case.

  \item Assumption \eqref{eq:convergence-general:weight-diff} ensures that the weights stay close to the initial value and is a standard estimate in NTK theory. 

  \item The interpolation inequality \eqref{eq:interpolation-inequality} and the coercivity \eqref{eq:convergence-general:coercive} are automatically satisfied if we define $\hs^s$ by eigen-decompositions of the NTK, as for the shallow $1d$ networks of the last section.

  \item Likewise, for the $1d$ shallow networks the in the previous section the initial concentration and perturbation Assumptions \eqref{eq:convergence-general:initial} and \eqref{eq:convergence-general:perturbation} are shown by the Fourier type eigen-decomposition of the NTK. For multidimensional and deep networks, the eigen-functions are given by spherical harmonics \cite{BiettiMairal2019,GeifmanYadavKastenGalunJacobsRonen2020,ChenXu2021}, so that some generalizations appear plausible.

  \item The result likely needs some tweaks to be applicable to deep networks. In particular, in assumption \eqref{eq:convergence-general:initial}, we assume that the partial derivative $\partial_r f_\theta$ depends only on $\theta_r$, but not on any other $\theta_k$ for $k \ne r$. While true for shallow networks, which are just a sum of independent terms, the assumption fails for deep ones. It is used to ensure independence of summands for some concentration inequalities and may require a layer-wise breakdown in the deep case.

\end{itemize}

\subsection{Overview of the Proof}
\label{sec:proof-overview}

\paragraph{The Neural Tangent Kernel Argument}

The proof follows a classical neural tangent kernel argument. An elementary computation shows that for any $S \in \real$ the $\|\cdot\|_S$-norm of the residual $\res := f_\theta - g$ evolves by
\begin{equation} \label{eq:norm-evolution}
  \frac{1}{2} \frac{d}{dt} \|\res\|_S^2
  = \dualp{\res, \frac{d\res}{dt} }_S
  = - \dualp{\res, (D f_\theta) (D f_\theta)^* \res}_S
  =: - \dualp{\res, H_\theta \, \res}_S,
\end{equation}
where $\theta=\theta(t)$ are the weights determined by the gradient flow. Our main interest is the case $S=0$, which provides the evolution of the loss $L(\theta) = \frac{1}{2} \|\res\|_0^2$. As we will see, we also need the case $S=s$ to control the regularity of the gradient flow.

The evolution equation looks like a linear problem, except that the operator $H_\theta = H_{\theta(t)}$ depends on the neural network weights and therefore changes over time. Nonetheless, this nonlinearity is weak in the neural tangent kernel regime, based on the following two core observations: First, the weights $\theta(t)$ do not move far from the initial value so that $\theta(t) \approx \theta(0)$ and hence
\begin{equation}
  H_{\theta(t)} \approx H_{\theta(0)}.
  \label{eq:ntk0}
\end{equation}
Second, at the random initial value $\theta(0)$ the operator $H_{\theta(0)}$ is close to its expectation
\begin{equation}
  H_{\theta(0)} \approx \E{H_{\theta(0)}} =: H,
  \label{eq:ntkE}
\end{equation}
where the last identity is the definition of the neural tangent kernel $H$. In conclusion, the evolution of the residual is a perturbation of the linear evolution equation
\begin{equation} \label{eq:evolution-norm-approx}
  \frac{1}{2} \frac{d}{dt} \|\res\|_S^2
  \approx - \dualp{\res, H \, \res}_S
  \lesssim - \|\res\|_{S-1}^2,
\end{equation}
where in the last equality we have used the coercivity of $H$.

\paragraph{Interpolation and Convergence of the Gradient Flow}

In the normal neural tangent kernel argument, the residual $\res$ measures the misfit of the neural network on $n$ sample points and is therefore contained in $\hs = \ell_2(\real^n)$. Since this space is finite dimensional, the norms $\|\cdot\|_S$ and $\|\cdot\|_{S-1}$ are equivalent. Thus the last equation ensures convergence by Gronwall's inequality as long as we can control the perturbation $H_{\theta(t)} \approx H$, which can be ensured by over-parametrization. However, in this paper we consider an under-parametrized regime with infinite data so that the residual $\res$ is contained in an infinite dimensional space $\hs$ and the norm equivalence no longer holds. Indeed, in the shallow case, if $\res$ has significant high frequency components, we have $\|\res\|_S \gg \|\res\|_{S-1}$ and thus for large loss $\|\res\|_S$ we only have a relatively small error decay $\frac{1}{2} \frac{d}{dt} \|\res\|_S^2 \approx - \|\res\|_{S-1}^2$.

To overcome this issue, assume for the moment $\|\res\|_s \le \gamma$ is bounded uniformly in $t$ so that we obtain the interpolation inequality
\[
  \|\res\|_0
  \le \|\res\|_{-1}^{\frac{s}{s+1}} \|\res\|_s^{\frac{1}{s+1}}
  \le \|\res\|_{-1}^{\frac{s}{s+1}} \gamma^{\frac{1}{s+1}}.
\]
Solving for $\|\res\|_{-1}$ and plugging in to the evolution \eqref{eq:evolution-norm-approx} with $S=0$, we obtain
\[
  \frac{1}{2} \frac{d}{dt} \|\res\|_0^2
  \lessapprox - \|\res\|_0^{2 \frac{s+1}{s}} \gamma^{-\frac{2}{s}},
\]
where $\lessapprox$ indicates missing perturbation terms. This time, the norm on the left hand side and right hand side are the same. Thus, if the residual is large, we also have a correspondingly large error decay and a Gronwall type inequality yields convergence.

\paragraph{Smoothness of the residual}

We still have to ensure that the residual $\|\res\|_s \le \gamma$ is smooth for all $t$. First, in order to obtain any quantifiable approximation error from a finite width neural network and $\hs$-norm loss, the target function $g$ must be contained in some compact set. Usually, these are balls in Sobolev and Besov spaces and in our case we assume that $\|g\|_s$ is bounded. Hence, a comparable bound for the neural network $\|f_\theta\|_s$ ensures smoothness of the residual $\res=f_\theta - g$. To this end, we again invoke the evolution equation \eqref{eq:evolution-norm-approx} with $S=s$ to obtain
\[
  \frac{1}{2} \frac{d}{dt} \|\res\|_s^2
  \approx - \|\res\|_{s-1}^2
  \lessapprox 0.
\]
Due to the perturbations $H_\theta(t) \approx H$, the right hand side can in fact be larger than zero, but as we will see, it remains sufficiently small so that we can still bound the smoothness $\|\res\|_s$ of the residual.

\paragraph{Neural Tangent Kernel Perturbations}

The bulk of the convergence proof is concerned with the two neural tangent kernel perturbations \eqref{eq:ntk0} and \eqref{eq:ntkE} and resolves two challenges. First, for the evolution of the $\|\res\|_s$-norm, we must control the perturbation of $H$ in related smoothness norms instead of simpler $\ell_2$ or $\hs$ norms. Second, $H$ maps $\hs$ to itself and is not a matrix in a finite dimensional space. Therefore, simple concentration arguments at the initialization like union bounds on the matrix entries are no longer permissible. Similar to \cite{OymakSoltanolkotabi2020,SongYang2019}, we use a dimension independent variant of the matrix Bernstein inequality instead.

\section{Numerical Experiments}
\label{sec:experiments}

Since the main result relies on a linearization argument, we consider several numerical experiments to see how gradient descent performs on an actual non-linear neural network. To this end, we consider three model functions, a Gaussian, a cusp and a jump:
\begin{align*}
	g(x) & = \frac{5}{4 \sqrt{2 \pi}}e^{-\frac{25}{2}x^2}, &
	g(x) & = 1 - \sqrt{|x|}, &
	g(x) & = \begin{cases} 
		0 & x\leq 0 \\
		1 & x > 0 
	\end{cases}.
\end{align*}
The Gaussian is smooth and can be approximated with high order by spline approximation. The cusp and jump are typical examples to compare uniform with free knot splines. While the former are closely related to the theory in this paper, the latter free knot approximation results can in theory be achieved by a globally optimal shallow neural network. A corresponding theory would require us to replace the smoothness norm \eqref{eq:shallow:s-norm} with a considerably weaker Besov norm situated roughly on the Sobolev embedding line through $L_2(\dom)$.

In order to account for the randomness in the initialization, all experiments are repeated one-hundred times with their statistics reported in the box-plots below. All experiments are conducted in Pytorch \cite{Pytorch2019} and the code is available on Github \cite{GentileWelper2022}. The Adam optimizer is selected en lieu of stochastic gradient descent since it yields reasonable results with minimal hyper-parameter tuning. Still, some parameters are adjusted in the interest of saving computational resources, depending on the target function and network size. An initial learning rate between $0.001$ and $0.02$ is used, and for some cases, learning rate decay is implemented via Pytorch's ``StepLR" function. Also, in the interest of computational efficiency, the number of training samples is a function of network size, $2m$, while a grid size of $100m$ is used for evaluating function norms.

\paragraph{Approximation Rates}

We first compare the approximation rates from Corollary \ref{cor:approximation-error-bias1d-short} with experiments. Figures \ref{fig:approximation-rate} and \ref{fig:approximation-rate-a-trained} show width-versus error plots for networks that are trained until no further significant improvement is observed and Table \ref{table:approximation-rate} shows the corresponding numerical order. Although the order of the Gaussian seems to change at some width, for all three examples the numerically determined orders are below the ones expected for uniform splines (orders $2$, $<1$, $<1/2$), free knot splines (orders $2$, $2$ and $\infty$) and above the ones predicted by our theory ($< 1/24$ in all cases). Within these upper and lower bounds, the order improves with the smoothness of the target functions. For better comparison with splines, we included both plots for trained and untrained outer layer $a_r$.

In conclusion, the provided theoretical results are not sharp, but correctly predict a loss in convergence rate compared to optimal rates. Hence, the networks have some redundancy in comparison to the achievable error for each given function. This seems to replace the redundancy in comparison to the number of samples in the usual over-parametrized regime.

\begin{figure}[h!]
	\includegraphics[width=\textwidth]{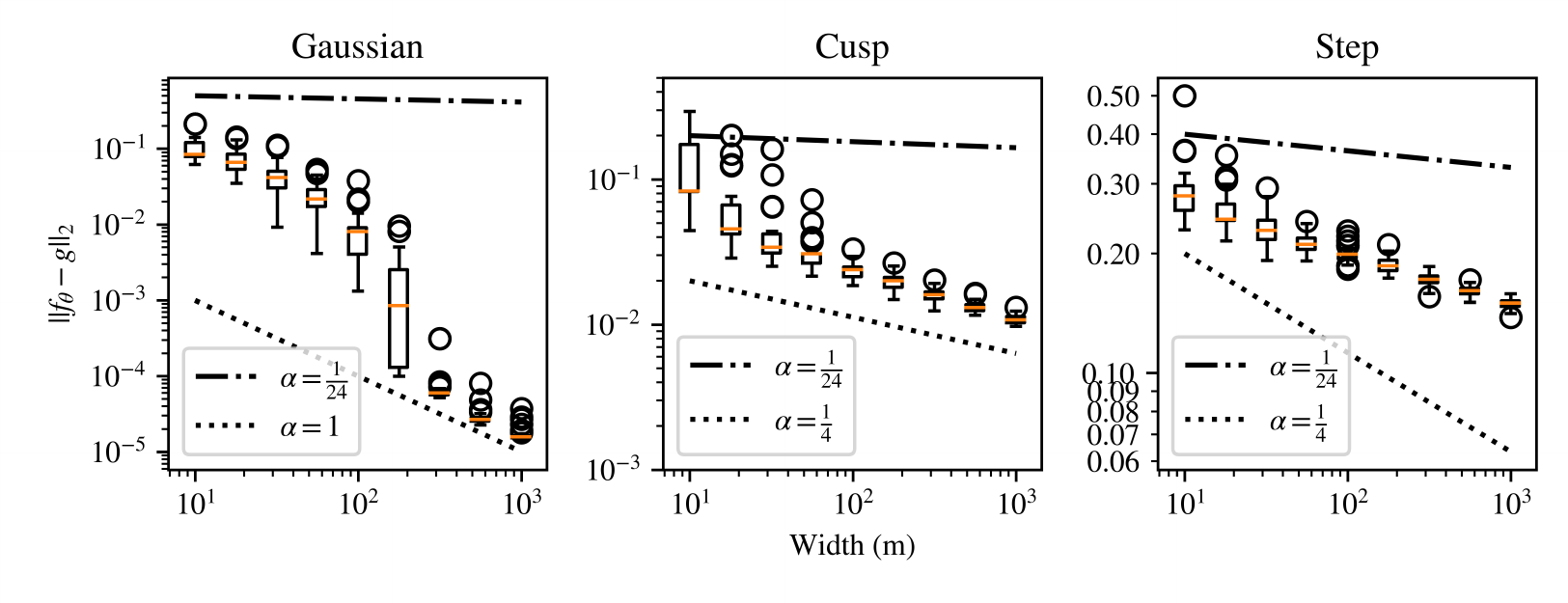}
	\centering
	\caption{$\|f_{\theta} - g\|_2$ Convergence with $a_r$ not trained.}
	\label{fig:approximation-rate}
\end{figure}

\begin{figure}[h!]
	\includegraphics[width=\textwidth]{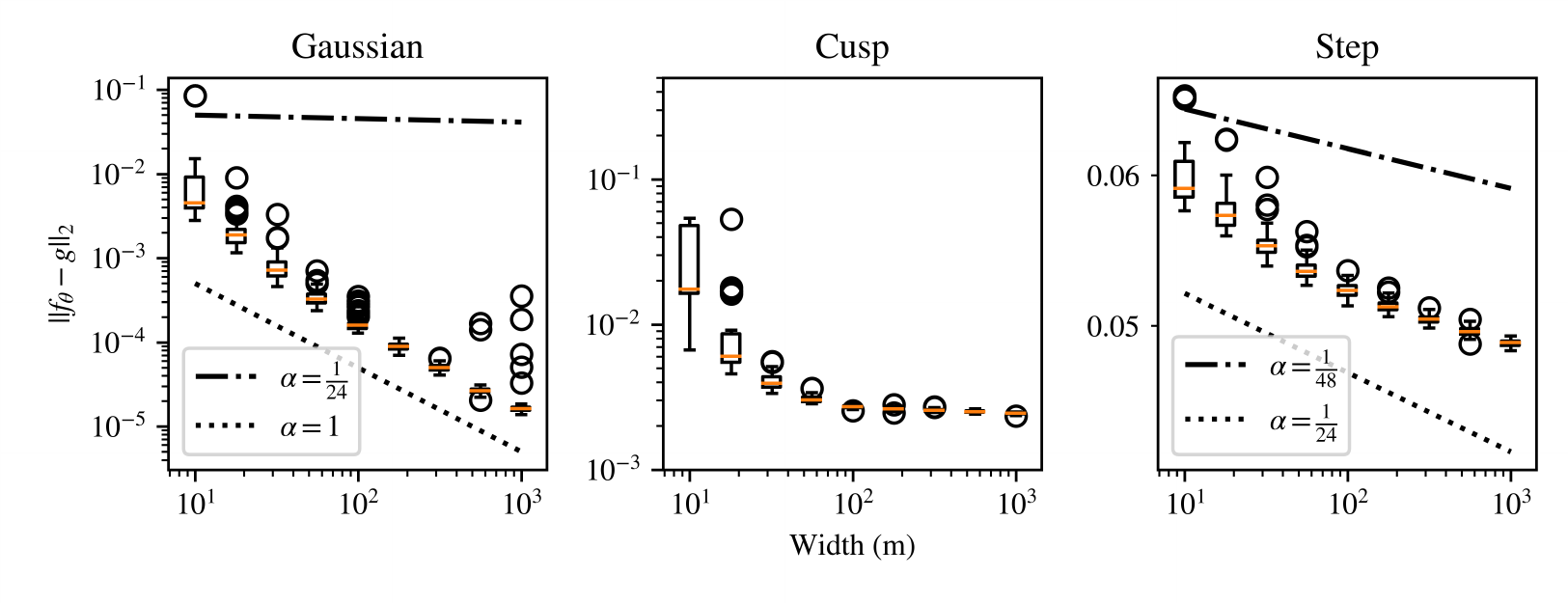}
	\centering
	\caption{$\|f_{\theta} - g\|_2$ Convergence with $a_r$ trained.}
	\label{fig:approximation-rate-a-trained}
\end{figure}

\begin{table}[h!]
	\centering
	\begin{tabular}{r|ccc|ccc}
		\toprule
		& \multicolumn{3}{c|}{$a_r$ Not Trained} & \multicolumn{3}{c}{$a_r$ Trained} \\
		m    &  Gaussian &  Cusp &  Step &  Gaussian &  Cusp &  Step \\
		\midrule
		18   & 0.62 & 0.97 & 0.16 & 2.76 & 1.98 & 0.08 \\
		32   & 0.93 & 0.87 & 0.18 & 1.66 & 1.13 & 0.07 \\
		56   & 1.01 & 0.57 & 0.13 & 1.52 & 0.51 & 0.05 \\
		100  & 1.94 & 0.42 & 0.12 & 1.27 & 0.21 & 0.04 \\
		178  & 2.95 & 0.32 & 0.11 & 1.07 & 0.06 & 0.03 \\
		316  & 5.16 & 0.39 & 0.14 & 1.01 & 0.03 & 0.03 \\
		562  & 1.66 & 0.33 & 0.11 & 0.94 & 0.04 & 0.03 \\
		1000 & 0.93 & 0.34 & 0.13 & 0.46 & 0.05 & 0.02 \\
		\bottomrule
	\end{tabular}
	\caption{Convergence rates of $\|f_{\theta} - g\|_2$.}
	\label{table:approximation-rate}
\end{table}

\paragraph{Gradient Descent Convergence}

Besides an approximation result, Theorem \ref{th:convergence-bias1d-short} contains error bounds for the gradient flow with respect to time $t$. Analogous gradient descent convergence plots for $m=1000$ are given in Figure \ref{fig:gd-convergence}. In each plot, the training process is shown for ten different random weight initializations, represented by different colors. These were selected at random from the full set of one-hundred repeated experiments. Given the simplicity of the target function, we need a considerable amount of iterations.

\begin{figure}[h!]
	\includegraphics[width=\textwidth]{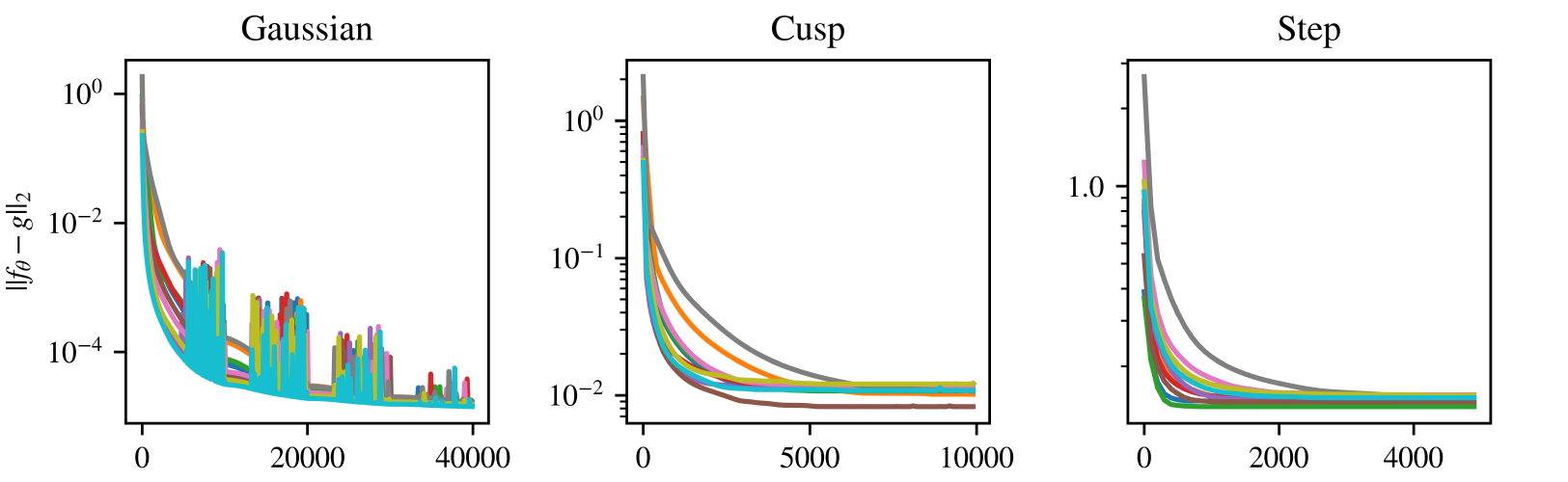}
	\centering
	\caption{Learning curves with $a_r$ not trained.}
	\label{fig:gd-convergence}
\end{figure}

\paragraph{Weight Change}

The NTK linearization argument relies on the fact that the weights do not move far from their initial values, at least for the limited training time in Theorem \ref{th:convergence-bias1d-short}. For later times, the theory does not provide any bounds and practical distances are shown in Figure \ref{fig:weight-distance}. Similar to Figure \ref{fig:gd-convergence}, ten samples, selected at random from the full set of one-hundred experiments, are shown in each plot. In all cases, the weights move out of the neighborhood used in the NTK theory. Nonetheless, Figure \ref{fig:weight-distance-rate} and Table \ref{table:weight-distance-rate} show that in correspondence with theory the weight distance is smaller for larger $m$. In contrast to the approximation rates, the rate at which the distance changes is fairly close to the rate $<1/3$ predicted by Corollary \ref{cor:approximation-error-initial}.

Note that in Figure \ref{fig:weight-distance} the weights for the step function never exceed $5$ during the training time, in contrast to $1/h$ of a natural hand crafted network
\[
f(x) = \frac{1}{h} \operatorname{ReLU}(x) - \frac{1}{h} \operatorname{ReLU}(x-h),
\]
which approximates the jump by a linear function between two breakpoints at distance $h$. Practical networks tend to achieve a similar function by stacking jumps of lesser height with smaller outer weights.

\begin{figure}[h!]
	\includegraphics[width=\textwidth]{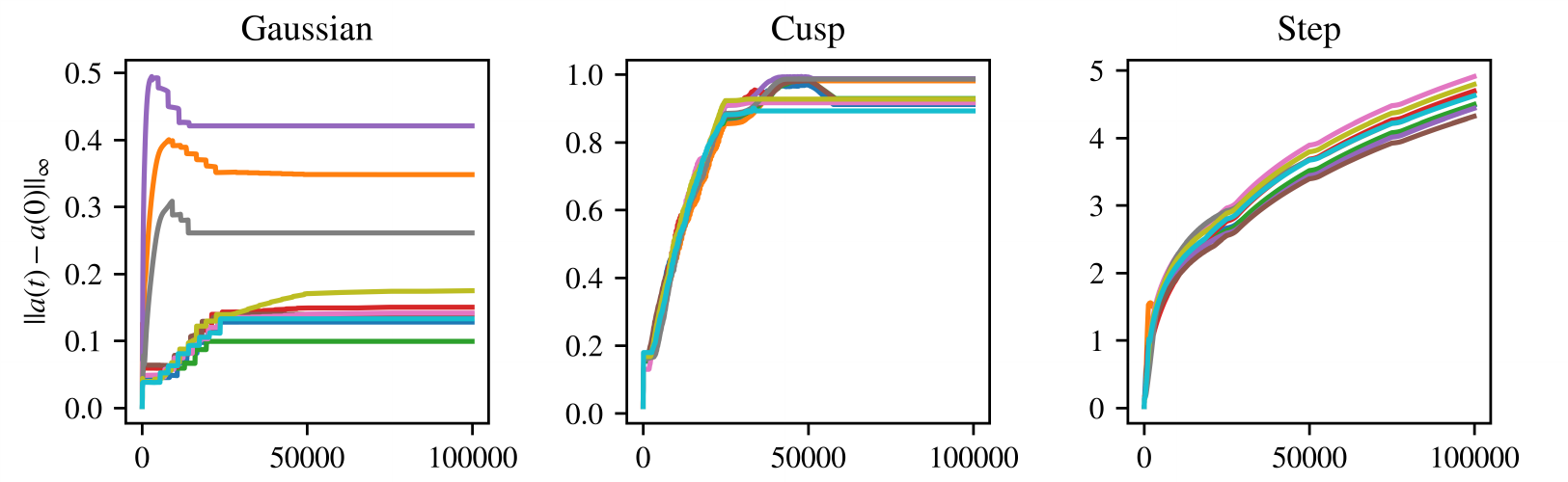}
	\centering
	\caption{Maximum weight change over time for large ($m=1000$) models.}
	\label{fig:weight-distance}
\end{figure}

\begin{figure}[h!]
	\includegraphics[width=\textwidth]{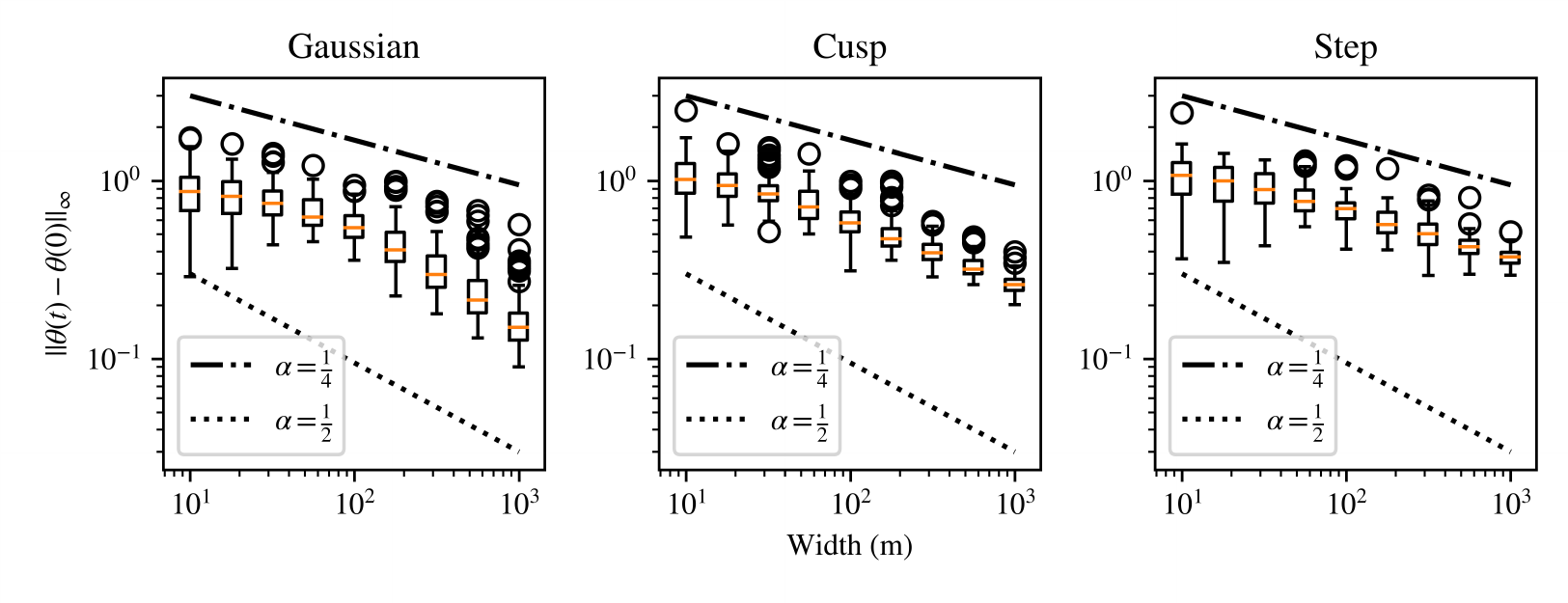}
	\centering
	\caption{$\|\theta(t) - \theta(0)\|_\infty$ convergence with $a_r$ not trained.}
	\label{fig:weight-distance-rate}
\end{figure}

\begin{table}[h!]
	\centering
	\begin{tabular}{r|ccc}
		\toprule
		m & Gaussian & Cusp &  Step \\
		\midrule
		18   & 0.10 & 0.17 & 0.11 \\
		32   & 0.14 & 0.17 & 0.11 \\
		56   & 0.22 & 0.28 & 0.20 \\
		100  & 0.34 & 0.39 & 0.28 \\
		178  & 0.38 & 0.29 & 0.29 \\
		316  & 0.55 & 0.40 & 0.25 \\
		562  & 0.49 & 0.34 & 0.31 \\
		1000 & 0.64 & 0.40 & 0.25 \\
		\bottomrule
	\end{tabular}
	\caption{Convergence rates of $\|\theta(t) - \theta(0)\|_\infty$ with $a_r$ not trained.}
	\label{table:weight-distance-rate}
\end{table}

\paragraph{Adaptive Grids}

As observed earlier, the theoretical results of this paper are comparable to uniform splines, but not free knot splines, although the latter are within the representation capabilities of the neural networks. This is somewhat a necessity of the NTK linearization: It requires that the biases $\theta$ do not move far from their initial values. But the biases are identical with the breakpoint locations, which therefore must also remain close to the initial uniform distribution, at least in the initial time frame where Theorem \ref{th:convergence-bias1d-short} applies. As we have seen, the weights leave this initial neighborhood eventually and Figures \ref{fig:breakpoint-histogram} and \ref{fig:breakpoint-histogram-a-trained} show histograms of their final location. Contrary to the NTK regime, they do show some reasonable adaptive behavior.

\begin{figure}[h!]
	\includegraphics[width=\textwidth]{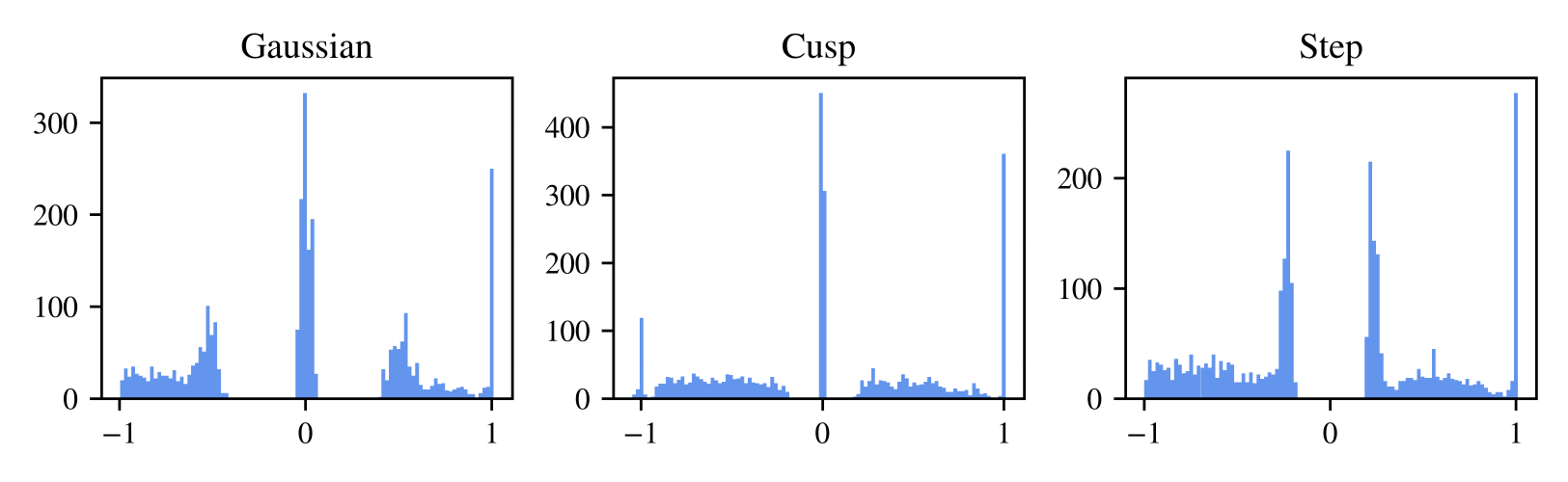}
	\centering
	\caption{Breakpoint histograms with $a_r$ not trained ($m=100$).}
	\label{fig:breakpoint-histogram}
\end{figure}

\begin{figure}[h!]
	\includegraphics[width=\textwidth]{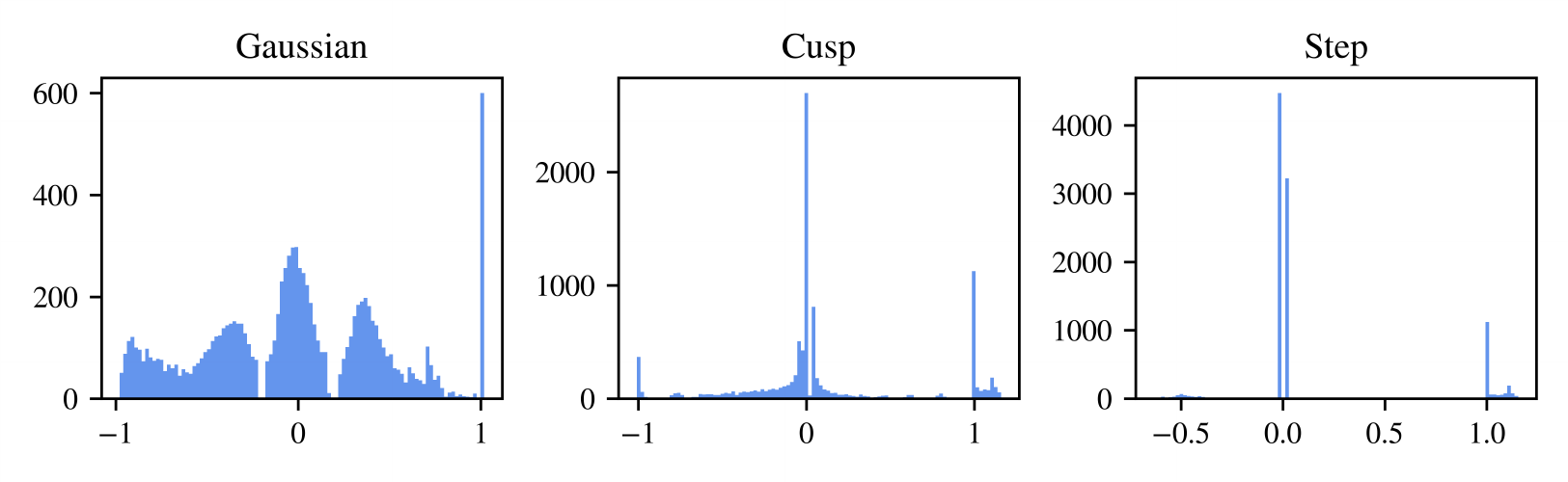}
	\centering
	\caption{Breakpoint histograms with $a_r$ trained ($m=100$).}
	\label{fig:breakpoint-histogram-a-trained}
\end{figure}

\section{Proof of Theorem \ref{th:convergence-general}}
\label{sec:proof-abstract}

In this section, we prove Theorem \ref{th:convergence-general}, following the idea outlined in the proof overview in Section \ref{sec:proof-overview}. The proof is organized as follows.

\begin{enumerate}
  \item Section \ref{sec:gradient-flow}: Derives the neural tangent kernel and the evolution \eqref{eq:norm-evolution} of the norms $\frac{1}{2} \frac{d}{dt} \|\res\|_S$.
  \item Section \ref{sec:convergence}: Shows a convergence result given that the derivatives $H_{\theta(t)} \approx H$ stay close to the expectation of the initial as outlined in \eqref{eq:ntk0} and \eqref{eq:ntkE}.
  \item Section \ref{sec:concentration}: Shows \eqref{eq:ntk0}, i.e., that the initial $H_{\theta(0)}$ is close to the expectation $H$.
  \item Section \ref{sec:perturbations}: Shows \eqref{eq:ntkE}, i.e., that $H_{\theta(t)}$ remains close to the initial $H_{\theta(0)}$ during training.
  \item Section \ref{eq:proof-combine}: Combines all partial results to the proof of Theorem \ref{th:convergence-general}.
\end{enumerate}

\subsection{Gradient Flow and the Neural Tangent Kernel}
\label {sec:gradient-flow}

We compute the evolution $\frac{d}{dt} \|\res\|_S^2$ of the residual $\res=f_\theta - g$. First note that the residual maps the weights $\theta \in \wdom$ to a function in $\hs$ and the loss maps the weights to a real value:
\begin{align*}
  \res: \wdom & \to \hs, & \theta & \to f_\theta - g, \\
  L: \wdom & \to \real, & \theta & \to L(\theta) = \frac{1}{2} \|\res\|_H^2.
\end{align*}
The weights $\theta = \theta(t)$ evolve by the gradient flow
\begin{equation*}
  \frac{d}{dt} \theta = - \nabla L(\theta).
\end{equation*}
In order to compute the right hand side, the Fr\'{e}chet derivative of the loss is
\begin{align*}
  D L(\theta) v
  & = \dualp{\res, (D f_\theta) v}
  = \dualp{ (D f_\theta)^* \res, v}, &
  & \text{for all }v \in \wdom
\end{align*}
and the gradient of the loss is the Riesz lift of the derivative
\begin{equation} \label{eq:gradient}
  \nabla L(\theta)
  = (D f_\theta)^* \res.
\end{equation}
Using the chain rule, we obtain
\begin{equation*}
  \frac{d\res}{dt}
  = (D f_\theta) \frac{d\theta}{dt}
  = - (D f_\theta) \nabla L(\theta)
  = - (D f_\theta) (D f_\theta)^* \res
  =: H_\theta \res
\end{equation*}
and the evolution of the norm
\begin{equation} \label{eq:gradien-flow}
  \frac{1}{2} \frac{d}{dt} \|\res\|_S^2
  = \dualp{\res, \frac{d\res}{dt} }_S
  = - \dualp{\res, (D f_\theta) (D f_\theta)^* \res}_S
  = - \dualp{\res, H_\theta \, \res}_S,
\end{equation}
where we have defined
\begin{align} \label{eq:ntk}
   H_\theta & := (D f_\theta) (D f_\theta)^*, &
   H & := E[H_\theta] = \int_\wdom H_\theta \, dP(\theta).
\end{align}
The second operator is the \emph{neural tangent kernel} and will play an essential role in our analysis, which is carried out in a regime where $H_\theta$ stays close to its expectation $H$.

It will be useful to split $H_\theta$ into a sum of partial derivatives. With the standard basis vectors $e_r$ of $\real^m$ the partial derivatives are given by
\begin{equation*}
  \partial_r f_\theta = (Df_\theta) e_r \in \hs
\end{equation*}
and thus, we can split
\begin{equation} \label{eq:ntk-partial-sum}
  H_\theta
  = (Df_\theta) (Df_\theta)^*
  = \sum_{r=1}^m \partial_r f_\theta \dualp{\partial_r f_\theta, \cdot}
  = \sum_{r=1}^m (\partial_r f_\theta) (\partial_r f_\theta)^*
\end{equation}
into a sum of rank one operators.

\subsection{Convergence}
\label{sec:convergence}

Using the notation $\|H\|_{0,S}$ for the induced operator norm of $H: \hs^0 \to \hs^S$, the following result provides the same conclusion as Theorem \ref{th:convergence-general}, which we aim to prove, but uses more elementary assumptions. These will be provided in the following sections.

\begin{lemma} \label{lemma:perturbed-gradient-flow}
  Assume we train weights of $f_\theta$ by a gradient flow according to the definitions in Sections \ref{sec:result-abstract} and \ref{sec:gradient-flow}. For $S \in \{0, s\}$, constants $c_\infty, c_0, \alpha > 0$ and functions $p_\infty(\tau)$, $p_0(m,\wdiff)$ assume that
  \begin{enumerate}

    \item \label{item:perturbed-gradient-flow:weight-distance} The distance of the weights from their initial value is controlled by
    \begin{equation} \label{eq:perturbed-gradient-flow:weight-diff}
      \|\theta(t) - \theta(0)\|_\infty
      \lesssim \sqrt{\frac{2}{m}} \int_0^t \|\res(\tau)\|_0 \, d\tau.
    \end{equation}

    \item \label{item:perturbed-gradient-flow:coercive} $H$ is coercive
    \begin{align} \label{eq:perturbed-gradient-flow:coercive}
      \|v\|_{S-1}^2 & \lesssim \dualp{v, H v}_S, &
      v & \in \hs^{S-1}
    \end{align}

    \item \label{item:perturbed-gradient-flow:initial}
    \begin{equation} \label{eq:perturbed-gradient-flow:initial}
      \pr{\|H - H_{\theta(0)}\|_{0,S} \ge \sqrt{\frac{c_\infty \tau}{m}}} \le p_\infty(\tau)
    \end{equation}
    for all $\tau$ with $\sqrt{\frac{c_\infty \tau}{m}} \le 1$.

    \item \label{item:perturbed-gradient-flow:perturbation}
    \begin{multline} \label{eq:perturbed-gradient-flow:perturbation}
      \pr{\exists \, \ptheta \in \wdom\text{ with }\|\ptheta - \theta(0)\|_\infty \le \wdiff\text{ and }\|H_\ptheta - H_{\theta(0)}\|_{0,S} \ge c_0 \wdiff^\alpha}
      \\
      \le p_0(m,\wdiff)
    \end{multline}
    for all $\wdiff > 0$.

  \end{enumerate}
  Then for $m$ sufficiently large so that $\frac{\tau}{m} \le 1$, with probability at least $1 - p_\infty(\tau) - p_0(m,\wdiff)$ we have
  \[
    \|\res(t)\|_0^2
    \lesssim \left[ \wdiff^\alpha \|\res(0)\|_s^{\frac{2}{s}} + \|\res(0)\|_0^{\frac{2}{s}} e^{-\wdiff^\alpha \frac{t}{s}} \right]^s
  \]
  with
  \begin{align*}
    \wdiff & = c_\wdiff^{\frac{1}{1+\alpha}} \|\res(0)\|_0^{\frac{1}{1+\alpha}} m^{-\frac{1}{2} \frac{1}{1+\alpha}}, &
    \tau & = c_0^{\frac{2\alpha}{1+\alpha}} \|\res(0)\|_0^{\frac{2\alpha}{1+\alpha}} m^{\frac{1}{1+\alpha}}
  \end{align*}
  for some constants $c_\wdiff,c,C \ge 0$ dependent of $s$ and independent of $\res$ and $m$.

\end{lemma}

\begin{proof}

For the time being, we assume that the weights remain within a finite distance 
\begin{equation} \label{eq:perturbed-gradient-flow:close-to-initial}
  \wdiff := \sup_{t \le T} \|\theta(t) - \theta(0)\|_\infty
\end{equation}
to their initial up to a time $T$ to be determined below. With this condition, we can bound the time derivatives of the loss $\|\res\|_0$ and the smoothness $\|\res\|_s$. For $S \in \{0,s\}$, we have already calculated the exact evolution in \eqref{eq:gradien-flow}, which we estimate by
\begin{align*}
  \frac{1}{2} \frac{d}{dt} \|\res\|_S^2
  & = - \dualp{ \res, H_{\theta(t)} \res }_S
  \\
  & = - \dualp{ \res, H \res }_S
      + \dualp{ \res, (H - H_{\theta(0)}) \res }_S
      + \dualp{ \res, (H_{\theta(0)} - H_{\theta(t)}) \res }_S
  \\
  & \le - \dualp{ \res, H \res }_S
        + \|H - H_{\theta(0)}\|_{0,S} \|\res\|_S \|\res\|_0
	+ \|H_{\theta(0)} - H_{\theta(t)}\|_{0,S} \|\res\|_S \|\res\|_0
  \\
  & \le - \dualp{ \res, H \res }_S
  + \left[ \sqrt{\frac{c_\infty \tau}{m}} + c_0 \wdiff^{\alpha} \right] \|\res\|_S \|\res\|_0,
  \\
  & \lesssim - \dualp{ \res, H \res }_S + \wdiff^{\alpha} \|\res\|_S \|\res\|_0,
\end{align*}
with probability at least $1-p_\infty(\tau) - p_0(m,\wdiff)$, where in the last inequality we have chosen $\tau = \wdiff^{2\alpha} m$ so that $\sqrt{\frac{c_\infty \tau}{m}} \lesssim \wdiff^\alpha$ and the second but last inequality follows from assumptions \eqref{eq:perturbed-gradient-flow:initial}, \eqref{eq:perturbed-gradient-flow:perturbation} and \eqref{eq:perturbed-gradient-flow:close-to-initial}. The left hand side contains one negative term $-\dualp{\res, H \res}_S$, which decreases the residual $\frac{d}{dt}\|\res\|_S^2$, and one positive term which enlarges it. In the following, we ensure that these terms are properly balanced. In case $S=s$, we only need to show that the smoothness $\|\res\|_s$ remains bounded for all relevant $t$ and therefore it is sufficient to consider the crude estimate
\[
  \dualp{ \res, H \res }_s
  \gtrsim \|\res\|_{s-1}^2
  \ge 0.
\]
For $S=0$, the norm $\|\res\|_0 = L(\theta)$ is the loss function, so we have to ensure that it becomes sufficiently small and thus we need a more careful estimate. Since we do not consider the time evolution of the $\|\res\|_{-1}$ norm, we lower bound it by the interpolation inequality $\|\res\|_0 \le \|\res\|_{-1}^{\frac{s}{s+1}} \|\res\|_s^{\frac{1}{s+1}}$ from \eqref{eq:interpolation-inequality}, which upon solving for $\|\res\|_{-1}$ yields
\[
  \dualp{ \res, H \res }_0
  \gtrsim \|\res\|_{-1}^2
  \ge \|\res\|_0^{2 + \frac{2}{s}} \|\res\|_s^{-\frac{2}{s}}.
\]
Thus, we arrive at the coupled system of differential inequalities
\begin{align*}
  \frac{1}{2} \frac{d}{dt} \|\res\|_0^2
  & \lesssim - \|\res\|_0^{2 + \frac{2}{s}} \|\res\|_s^{-\frac{2}{s}} + \wdiff^\alpha \|\res\|_0^2
  \\
  \frac{1}{2} \frac{d}{dt} \|\res\|_s^2
  & \lesssim \wdiff^{\alpha} \|\res\|_s \|\res\|_0.
\end{align*}
Notice that by means of the interpolation inequality this system does not involve any $\|\res\|_S$ norms for $S \not\in \{0,s\}$. Since $\|\cdot\|_0 \le \|\cdot\|_s$, by Lemma \ref{lemma:ode-system-bounds} with $x = \|\res\|_0^2$, $y = \|\res\|_s^2$ and $\rho = \frac{1}{s}$, we have
\[
  \|\res\|_s^2
  \lesssim \left[ \|\res(0)\|_s + \frac{\wdiff^{\alpha}}{\wdiff^{\alpha}} \|\res(0)\|_0 \right]^2
  \lesssim \|\res(0)\|_s^2,
\]
and
\begin{equation} \label{eq:perturbed-gradient-flow:bound}
  \|\res\|_0^2
  \lesssim \left[ \wdiff^\alpha \|\res(0)\|_s^{\frac{2}{s}} + \|\res(0)\|_0^{\frac{2}{s}} e^{-c\wdiff^\alpha \frac{t}{s}} \right]^s
  \lesssim \|\res(0)\|_0^2 e^{-c\wdiff^\alpha t},
\end{equation}
for some generic constant $c>0$, where the last inequality holds as long as the second summand dominates, i.e.
\begin{equation} \label{eq:perturbed-gradient-flow:time-limit}
  \wdiff^\alpha \|\res(0)\|_s^{\frac{2}{s}} \le \|\res(0)\|_0^{\frac{2}{s}} e^{-c\wdiff^\alpha \frac{t}{s}}.
\end{equation}
defining a final time $T$ when the sizes flip. This shows the statement of the lemma up to a computation of $\wdiff$, defined in \eqref{eq:perturbed-gradient-flow:close-to-initial}, and a proof that it remains finite. We first compute $\wdiff$. By Assumption \eqref{eq:perturbed-gradient-flow:weight-diff} and \eqref{eq:perturbed-gradient-flow:bound} we have
\begin{multline} \label{eq:wdiff}
  h 
  = \sup_{t \le T}\|\theta(t) - \theta(0)\|_\infty
  \lesssim \sqrt{\frac{2}{m}} \int_0^T \|\res(\tau)\|_0 \, d\tau
  \\
  \lesssim \sqrt{\frac{2}{m}} \|\res(0)\|_0 \int_0^T e^{- c\wdiff^\alpha \frac{\tau}{2}} \, d\tau
  \le c \sqrt{\frac{1}{m}} \frac{\|\res(0)\|_0}{\wdiff^\alpha},
\end{multline}
for some generic constant $c>0$. Solving for $\wdiff$, we obtain 
\[
\wdiff \le B := c^{\frac{1}{1+\alpha}} \|\res(0)\|_0^{\frac{1}{1+\alpha}} m^{-\frac{1}{2} \frac{1}{1+\alpha}}.
\]
Note that the use of \eqref{eq:perturbed-gradient-flow:bound} implicitly assumes that $h$ is bounded, which we show by contradiction. Indeed, if it is unbounded, let $\bar{T}$ be the smallest time when the weights separate farther than $2B$
\[
  \bar{T} := \inf \{0 \le t \le T \, | \, \|\theta(t) - \theta(0)\|_\infty \ge 2B \}.
\]
For any $t<\bar{T}$, we can use the arguments above to see that $\|\theta(t) - \theta(0)\|_\infty \le B$ on the contrary, by definition for some $t \ge \bar{T}$ arbitrarily close to $\bar{T}$ we have $\|\theta(t) - \theta(0)\|_\infty \ge 2B$. This contradicts the continuity of $\theta(t)$ so that $\wdiff$ must be bounded.

Finally, we extend the result beyond the largest time $T$ for which \eqref{eq:perturbed-gradient-flow:time-limit} is satisfied and $\wdiff^\alpha \|\res(0)\|_s^{\frac{2}{s}} = \|\res(0)\|_0^{\frac{2}{s}} e^{-\wdiff^\alpha \frac{T}{s}}$. Since $\|\res\|_0^2$ is defined by a gradient flow, it is monotonically decreasing and thus for any time $t > T$, we have
\[
  \|\res(t)\|_0^2 \le \|\res(T)\|_0^2
  \lesssim \left[ 2 \wdiff^\alpha \|\res(0)\|_s^{\frac{2}{s}} \right]^s
  \lesssim 2^s \left[ \wdiff^\alpha \|\res(0)\|_s^{\frac{2}{s}} + \|\res(0)\|_0^{\frac{2}{s}} e^{-\wdiff^\alpha \frac{t}{s}} \right]^s
\]
so that the error bound \eqref{eq:perturbed-gradient-flow:bound} holds for all times up to an adjustment of the constants. This implies the statement of the lemma with our choice of $\wdiff$ and $\tau$.

\end{proof}

\subsection{Concentration at the Initial Value}
\label{sec:concentration}

This section contains a concentration result that shows $H_{\theta(0)} \approx H$. To this end, recall from \eqref{eq:ntk-partial-sum} that $H_{\theta(0)}$ is a sum of independent random rank one operators. Hence we use matrix concentration inequalities to show that $H_{\theta(0)}$ concentrates close to its expectation, with high probability. See \cite{OymakSoltanolkotabi2020,SongYang2019} for some similar results.

\begin{lemma} \label{lemma:initial-concentration}

  Assume that the partial derivatives $\partial_r f_\theta$, $r=1, \dots, m$ only depend on the single weight $\theta_r$ and that $\|\partial_r f_\theta\|_S \le \frac{\mu}{\sqrt{m}}$ for $S \in \{s, t\}$, $s,t \in \real$. Then for independently sampled initial weights $\theta_r$ and all $\tau > 0$, we have
  \[
    \pr{ \|H_\theta - H\|_{-t,s} \ge \sqrt{\frac{8 \mu^4 \tau}{m}} + \frac{2 \mu^2 \tau}{3m}}
    \le 2 \tau \left(e^\tau - \tau - 1 \right)^{-1}.
  \]

\end{lemma}

\begin{proof}

$H_\theta$ is a sum of independent random operators so that we can compute its concentration by matrix Bernstein inequalities. Indeed, by \eqref{eq:ntk-partial-sum}, we can decompose $H_\theta$ as 
\begin{align*}
  H_\theta & = \frac{1}{m} \sum_{r=1}^m X_r(\theta_r), &
  X_r(\theta_r) & := v_r u_r^*, &
  v_r & := \sqrt{m} \partial_r f_\theta, &
  u_r^* & := \sqrt{m} \dualp{\partial_r f_\theta, \cdot}.
\end{align*}
Since by assumption the partial derivatives $\partial_r f_\theta$ only depend on a single component $\theta_r$, all random operators $X_r(\theta_r)$ depend only on a single $\theta_r$, and are therefore independent. Moreover, by assumption we have $\|v_r\|_s \le \mu$ and $\|u_r^*\|_{-t,*} = \|u_r\|_t \le \mu$, where $\|\cdot\|_{-t,*}$ denotes the norm of the dual space $(\hs^{-t})^*$. Thus, by Corollary \ref{cor:matrix-concentration-rank-one} of the matrix Bernstein inequality stated below, for all $\tau \ge 0$
\[
  \pr{
    \left\| \frac{1}{m} \sum_{r=1}^m X_r - \E{X_r} \right\|_{-t,s}
    >
    \sqrt{ \frac{8 \mu^4 \tau}{m}} + \frac{2 \mu^2 \tau}{3m}
  }
  \le
  2 \tau \left( e^\tau - \tau - 1 \right)^{-1}.
\]
By construction, we have $\frac{1}{m} \sum_{r=1}^m X_r = H_\theta$ and $\frac{1}{m} \sum_{r=1}^m \E{X_r} = \E{H_\theta} = H$ so that
\[
  \left\| \frac{1}{m} \sum_{r=1}^m X_r - \E{X_r} \right\|_{-t,s}
  = \|H_\theta - H\|_{-t,s}
\]
which completes the proof.

\end{proof}

\subsection{Perturbations}
\label{sec:perturbations}

We prove that $H_\theta \approx H_\ptheta$ for $\theta \approx \ptheta$, which is used to show that $H_{\theta(t)} \approx H_{\theta(0)}$ given that the weights do not differ much from their initial values $\theta(t) \approx \theta(0)$.

\begin{lemma} \label{lemma:perturbation}
  Let $S \in \{s,t\}$, $s,t \in \real$ and $\theta, \ptheta \in \wdom$. Assume that
  \begin{align*}
    \|\partial_r f_\theta\|_S & \le \frac{\mu}{\sqrt{m}}, &
    \sup_{\|\nu\|_\infty \le 1} \left\| \sum_{r=1}^m (\partial_r f_\theta - \partial_r f_\ptheta) \nu_r \right\|_S & \le \sqrt{m} L \|\theta - \ptheta\|_\infty^\alpha
  \end{align*}
  for some $\alpha, \mu, L \ge 0$. Then
  \[
    \| H_\theta - H_\ptheta\|_{-t,s}
    \le 2 \mu L \|\theta - \ptheta\|_\infty^\alpha.
  \]
\end{lemma}

\begin{proof}

Using the split \eqref{eq:ntk-partial-sum} of $H_\theta$ into a sum of the partial derivatives, we obtain
\begin{align*}
  \|H_\theta - H_\ptheta\|_{-t,s}
  & = \left\| \sum_{r=1}^m (\partial_r f_\theta) (\partial_r f_\theta)^* - (\partial_r f_\ptheta) (\partial_r f_\ptheta)^*\right\|_{-t,s}
  \\
  & \le \left\| \sum_{r=1}^m (\partial_r f_\theta - \partial_r f_\ptheta) (\partial_r f_\theta)^*\right\|_{-t,s}
  \\
  & \quad + \left\| \sum_{r=1}^m (\partial_r f_\ptheta) (\partial_r f_\theta - \partial_r f_\ptheta)^*\right\|_{-t,s}
  \\
  & =: S_1 + S_2
\end{align*}
For the first summand $S_1$, we replace the operator norm by its definition so that for any $\wdiff > 0$, we have
\begin{equation*}
  S_1
  = \sup_{v \ne 0} \left\| \sum_{r=1}^m (\partial_r f_\theta - \partial_r f_\ptheta) \frac{(\partial_r f_\theta)^* v}{\wdiff} \right\|_s \frac{\wdiff}{\|v\|_{-t}}
\end{equation*}
With $\nu = \nu(v) \in \real^m$ defined by $\nu_r := \frac{1}{\wdiff} (\partial_r f_\theta)^* v$ and $\wdiff := \max_r |(\partial_r f_\theta)^* v| = \max_r |\dualp{\partial_r f_\theta, v}| \le \max_r \|\partial_r f_\theta\|_t \|v\|_{-t}$ this simplifies to
\begin{equation*}
  S_1
  \le \sup_{v \ne 0} \left\| \sum_{r=1}^m (\partial_r f_\theta - \partial_r f_\ptheta) \nu_r \right\|_s \max_r \|\partial_r f_\theta\|_t.
\end{equation*}
Our choice of $\wdiff$ implies $|\nu_r| \le 1$ for all $r$ and thus
\begin{equation*}
  S_1
  \le \sup_{\|\nu\|_\infty \le 1} \left\| \sum_{r=1}^m (\partial_r f_\theta - \partial_r f_\ptheta) \nu_r \right\|_s \max_r \|\partial_r f_\theta\|_t.
\end{equation*}
Since we define adjoints $(\cdot)^*$ is with respect to the $\hs_0 = L_2(\dom)$ scalar product and not $\hs^s$ and $\hs^t$, we have $\|H_\theta\|_{-t,s} = \|H_\theta^*\|_{-s,t}$ and thus an analogous argument for the second summand yields
\begin{equation*}
  S_2
  \le \sup_{\|\nu\|_\infty \le 1} \left\| \sum_{r=1}^m (\partial_r f_\theta - \partial_r f_\ptheta) \nu_r \right\|_t \max_r \|\partial_r f_\ptheta\|_s.
\end{equation*}
Plugging the assumptions into the estimates for $S_1$ and $S_2$ yields the result.

\end{proof}

If all partial derivatives are H\"{o}lder continuous, by the triangle inequality, we obtain the following simplified result, which is not sharp.

\begin{lemma}
  Assume that for $S \in \{s,t\}$, $s,t \in \real$ and all $\theta, \ptheta \in \wdom$
  \begin{align*}
    \|\partial_r f_\theta\|_S & \le \frac{\mu}{\sqrt{m}}, &
    \|\partial_r f_\theta - \partial f_\ptheta\|_S & \le \frac{L}{\sqrt{m}} \|\theta - \ptheta\|_\infty^\alpha
  \end{align*}
  for some $\alpha, \mu, L \ge 0$. Then
  \[
    \| H_\theta - H_\ptheta\|_{-t,s}
    \le 2 \mu L \|\theta - \ptheta\|_\infty^\alpha.
  \]
\end{lemma}

\begin{proof}

By the triangle inequality, we have for any $\|\nu\|_\infty \le 1$ that
\[
  \left\| \sum_{r=1}^m (\partial_r f_\theta - \partial_r f_\ptheta) \nu_r \right\|_S
  \le \sum_{r=1}^m \left\| (\partial_r f_\theta - \partial_r f_\ptheta) \right\|_S |\nu_r|
  \le \sqrt{m} L \|\theta - \ptheta\|_S^\alpha.
\]
Thus, the result follows from Lemma \ref{lemma:perturbation}.

\end{proof}

The triangle inequality in the last proof is a rather crude estimate. Indeed, if the partial derivatives are close to orthogonal in the $S$ norms, one can alternatively estimate
\begin{equation} \label{eq:perturbation-orthogonal}
  \left\| \sum_{r=1}^m (\partial_r f_\theta - \partial_r f_\ptheta) \nu_r \right\|_S
  \lesssim \left( \sum_{r=1}^m \left\| (\partial_r f_\theta - \partial_r f_\ptheta) \right\|_S^2 |\nu_r|^2 \right)^{1/2}
  \le L \|\theta - \ptheta\|_S^\alpha,
\end{equation}
an improvement by a factor $\sqrt{m}$. In spaces like $L_2$, orthogonality is implied by disjoint local supports. Although that is too strong for our application, variants of this argument are sill applicable. We first have to ensure that the $\hs^s$ norms have similar localization properties than $L_2$: We say that a norm \emph{localizes} if for any cover $\mathcal{P}$ of $\dom$ with bounded number of overlaps $n_o = \max_x |\{P \in \mathcal{P} | x \in P\}|$ we have
\begin{equation} \label{eq:norm-local}
  \left\|\sum_{P \in \mathcal{P}} f_P \right\|^2
  \lesssim n_o \sum_{P \in \mathcal{P}} \|f_P\|^2
\end{equation}
for all functions $f_P$ with support $\supp(f_P) \subset P$. This is clearly satisfied for $L_2$ and Sobolev norms and in Section \ref{sec:shallow:perturbation}, we will see that the $\hs^s$ norms do satisfy some related properties. 

As we will see below, with high probability the partial derivatives $f_P = \partial_r f_\theta - \partial_r f_\ptheta$ have local support with bounded overlap, so that the localization property of the norms directly leads to a slightly weakened variant of the orthogonality argument \eqref{eq:perturbation-orthogonal}.

\subsection{Proof of the Theorem}
\label{eq:proof-combine}

\begin{proof}[Proof of Theorem \ref{th:convergence-general}]

The Theorem follows directly from Lemma \ref{lemma:perturbed-gradient-flow}, once we verify all assumptions.

\begin{enumerate}

  \item \emph{Weight Distance:} The bounds for $\|\theta(t)-\theta(0)\|_\infty$ in assumptions \eqref{eq:convergence-general:weight-diff} for the Theorem and \eqref{eq:perturbed-gradient-flow:weight-diff} for the Lemma are identical.

  \item \emph{Coercivity:} The coercivity assumptions \eqref{eq:convergence-general:coercive} for the Theorem and \eqref{eq:perturbed-gradient-flow:coercive} for the Lemma are identical.

  \item \emph{Initial concentration:} Assumption \eqref{eq:convergence-general:initial} together with Lemma \ref{lemma:initial-concentration} imply that
  \[
    \pr{ \|H_\theta - H\|_{0,S} \ge \sqrt{\frac{8 \mu^4 \tau}{m}} + \frac{2 \mu^2 \tau}{3m}}
    \le 2 \tau \left(e^\tau - \tau - 1 \right)^{-1}.
  \]
  For $m$ sufficiently large so that $\frac{\tau}{m} \le 1$, we have $\left(\sqrt{8} + \frac{2}{3}\right) \mu^2 \sqrt{\frac{\tau}{m}} \gtrsim \sqrt{\frac{8 \mu^4 \tau}{m}} + \frac{2 \mu^2 \tau}{3m}$, which directly implies assumption \eqref{eq:perturbed-gradient-flow:initial} of Lemma \ref{lemma:perturbed-gradient-flow} with $\sqrt{c_\infty} = \left(\sqrt{8} + \frac{2}{3}\right) \mu^2$ and $p_\infty(\tau) = 2 \tau \left(e^\tau - \tau - 1 \right)^{-1}$.

  \item \emph{Perturbation:} In the random event that \eqref{eq:convergence-general:perturbation} holds, assumption \eqref{eq:convergence-general:initial} together with Lemma \ref{lemma:perturbation} yield
  \[
    \| H_\theta - H_\ptheta\|_{-t,s}
    \le 2 \mu L \|\theta - \ptheta\|_\infty^\alpha
  \]
  and thus provide assumption \eqref{eq:perturbed-gradient-flow:perturbation} of Lemma \ref{lemma:perturbed-gradient-flow}.

\end{enumerate}
Thus, the theorem follows from Lemma \ref{lemma:perturbed-gradient-flow}.

\end{proof}

\section{Proof of Theorem \ref{th:convergence-bias1d-short}}
\label{sec:proof-1d}

The theorem is a consequence of the general variant Theorem \ref{th:convergence-general}. In this section, we verify all assumptions for the shallow case. First, we re-state Theorem \ref{th:convergence-bias1d-short} in a more rigorous variant, where all constants are fully independent of $\res$ and in particular no longer depend on $\|\res\|_0$.

\begin{theorem} \label{th:convergence-bias1d}
  Let $\res = \res(t) = f_{\theta(t)} - g$ be the residual of the gradient flow \eqref{eq:shallow:gradient-flow} for the shallow network \eqref{eq:nn}. Then for the smoothness norms $\|\cdot\|_s$ defined in \eqref{eq:shallow:s-norm}, for $0 < s < \frac{1}{2}$ and $m \ge c_m$, with probability at least $1 - C_\tau \tau (e^\tau - \tau - 1)^{-1} - \frac{C_\wdiff}{\wdiff} e^{-2 m \wdiff^2}$ we have
  \[
    \|\res(t)\|_0^2
    \le C \left[ \wdiff^{1-s} \|\res(0)\|_s^{\frac{2}{s}} + \|\res(0)\|_0^{\frac{2}{s}} e^{-\wdiff^{1-s} \frac{t}{s}} \right]^s
  \]
  with
  \begin{align*}
    \wdiff & = c_\wdiff^{\frac{1}{2-s}} \|\res(0)\|_0^{\frac{1}{2-s}} m^{-\frac{1}{2} \frac{1}{2-s}}, &
    \tau & = c_0^{2 \frac{1-s}{2-s}} \|\res(0)\|_0^{2 \frac{1-s}{2-s}} m^{\frac{1}{2-s}}
  \end{align*}
  for some constants $c_m, c_\wdiff, c_0, C, C_\wdiff, C_\tau \ge 0$ dependent of $s$ and independent of $\res$ and $m$.

\end{theorem}

The proof is organized as follows.

\begin{enumerate}
  \item Section \ref{sec:shallow:ntk}: Provides explicit formulas for the NTK of the shallow network.
  \item Section \ref{sec:shallow:ev}: Computes the eigenvalues and eigenvectors of the NTK.
  \item Section \ref{sec:shallow:weight-dist}: Shows that the weights do not move far from the initial.
  \item Section \ref{sec:shallow:derivative-bounds}: Provides bounds for the partial derivatives and their H\"{o}lder continuity used for the assumption of Theorem \ref{th:convergence-general}.
  \item Section \ref{eq:shallow:proof-combine}: Combines all partial results to the proof of Theorem \ref{th:convergence-bias1d}.
\end{enumerate}

\subsection{The Neural Tangent Kernel}
\label{sec:shallow:ntk}

We compute the neural tangent kernel for the shallow network \eqref{eq:nn}. To this end, we first need the Fr\'{e}chet derivatives
\begin{equation} \label{eq:shallow:derivatives}
  \begin{aligned}
    (Df_\theta) u & = x \to \frac{1}{\sqrt{m}} \sum_{r=1}^m a_r \sigma'(x-\theta_r) u_r \quad \in \hs, & u & \in \wdom
    \\
    (Df_\theta)^* v & = r \to \frac{1}{\sqrt{m}} \int_\dom a_r \sigma'(y-\theta_r) v(y) \, dy \quad \in \wdom, & v & \in \hs.
  \end{aligned}
\end{equation}
Since $a_r = \pm 1$, we have $a_r^2 = 1$ so that definition \eqref{eq:ntk} of $H_\theta$ implies
\begin{equation*}
  H_\theta v = x \to \frac{1}{m} \sum_{r=1}^m \int_\dom \sigma'(x-\theta_r) \sigma'(y-\theta_r) v(y) \, dy \quad \in \hs.
\end{equation*}
Recall that the neural tangent kernel is defined by $H = \E{Df_\theta (Df_\theta)^*} = \E{H_\theta}$ at $\theta=\theta(0)$. Since the latter is uniformly distributed on $\dom=[-1,1]$, we obtain
\begin{equation} \label{eq:ntk-bias}
  H v = x \to \frac{1}{2} \int_\dom \int_\wdom \sigma'(x-\theta) \sigma'(y-\theta) v(y) \, d\theta \, dy \quad \in \hs.
\end{equation}
This operator integrates $r$ twice, with the unit step functions $\sigma'(x-\theta)$ and $\sigma'(y-\theta)$ restricting the integral bounds to subdomains $\theta \in [-1, x]$ and $y \in [\theta, 1]$ of $\dom = [-1, 1]$, respectively. Hence, the operator simplifies to
\begin{align} \label{eq:ntk-bias-simplified}
  (H v)(x) & = \frac{1}{2} \int_{-1}^x V(\theta) \, d\theta, &
  V(\theta) & = \int_\theta^1 v(y) \, dy.
\end{align}

\subsection{Eigenvalues and Eigenvectors}
\label{sec:shallow:ev}

We first show that the neural tangent kernel of \eqref{eq:nn} is the inverse of a Laplacian with suitable boundary conditions. Then, we use standard arguments to compute the eigenvalues and eigenvectors explicitly.

\begin{lemma} \label{lemma:ntk-inverse-laplace}

  Let $v \in \hs$ and $H$ be the neural tangent kernel defined in \eqref{eq:ntk-bias}. Then $Hv$ is the unique solution of the boundary value problem
  \begin{align*}
    -\frac{d^2}{dx^2} (Hv)(x) & = \frac{1}{2} v(x), &
    (Hv)(-1) & = 0, &
    \frac{d}{dx}(Hv)(1) & = 0
  \end{align*}
  for all $x \in \dom$.

\end{lemma}

\begin{proof}

Recall from \eqref{eq:ntk-bias-simplified} that the neural tangent kernel is
\begin{align*}
  (H v)(x) & = \frac{1}{2}\int_{-1}^x V(\theta) \, db, &
  V(\theta) & = \int_\theta^1 v(y) \, dy.
\end{align*}
Differentiating twice, we see that $Hv$ satisfies the Poisson equation
\[
  -\frac{d^2}{dx^2} (Hv)(x)
  = - \frac{1}{2}\frac{d}{dx} V(x)
  = \frac{1}{2} v(x),
\]
with no minus sign in the last step because $x$ is the lower integral bound in the definition of $V(x)$. Differentiating once and evaluating at $x=1$, we obtain the right boundary condition
\[
  -\frac{d}{dx} (Hv)(1)
  = - \frac{1}{2} V(1)
  = 0.
\]
Finally $(Hv)(-1) = 0$ is implied by the integral bounds.

\end{proof}

With the last lemma, it is standard to compute the eigenvalues and eigenvectors of $H$.

\begin{lemma} \label{lemma:shallow:eigenbasis}

  Define
  \begin{align*}
    \phi_k(x) & = \sin(\omega_k x - \varphi_k), &
    \omega_k & = \frac{\pi}{4} + \frac{\pi}{2} k, &
    \varphi_k & = - (-1)^k \frac{\pi}{4}.
  \end{align*}
  Then $\phi_k$ is an $\hs$ orthonormal basis of eigenfunctions of $H$ with corresponding eigenvalues $2 \omega_k^{-2}$.

\end{lemma}

\begin{proof}

By Lemma \ref{lemma:ntk-inverse-laplace} the space $\hs$ is boundedly mapped to the Sobolev space $H^2(\dom)$ by $H$ so that the operator is compact. Since it is also symmetric, by the spectral theorem it has an orthogonal basis of eigenvectors.

In order to avoid the factor $1/2$ in the definition of $H$, we compute the eigenvalues of $\bar{H} := \frac{1}{2} H$. Differentiating the eigenvalue equation $\bar{H} v = \omega^{-2} v$ twice and using Lemma \ref{lemma:ntk-inverse-laplace}, we obtain
\[
  \omega^{-2} v'' = (\bar{H}v)'' = - v
  \quad \Leftrightarrow \quad
  v'' + \omega^2 v = 0,
\]
so that all eigenfunctions must be of the form
\[
  v(x) = \sin(\omega x - \varphi).
\]
The constants $\omega$ and $\varphi$ are chosen so that the boundary conditions of Lemma \ref{lemma:ntk-inverse-laplace} are satisfied:
\begin{align*}
  0 & = \omega^2 (Hv)(-1) = v(-1) = \sin(-\omega - \varphi), & & \Leftrightarrow & & \omega + \varphi = \pi m
  \\
  0 & = \omega^2 (Hv)'(1) = v'(1) = \omega \cos(\omega - \varphi), & & \Leftrightarrow & & \omega - \varphi = \frac{\pi}{2} + \pi n
\end{align*}
for some integer $m$ and $n$ and in the respective second equations we have used that $v$ is an eigenfunction. Adding and subtracting the last two equations, we obtain
\begin{equation*}
  \begin{aligned}
    \omega
    & = \phantom{-} \frac{\pi}{4} + \frac{\pi}{2} (m + n)
    = \phantom{-} \frac{\pi}{4} + \frac{\pi}{2} k
    =: \omega_k
    \\
    \varphi
    & = - \frac{\pi}{4} + \frac{\pi}{2} (m - n)
    = - \frac{\pi}{4} + \frac{\pi}{2} (k - 2n),
  \end{aligned}
\end{equation*}
where we have defined $k := m+n$ and eliminated $m$ from the last equation by $m = k-n$ and therefore $m-n = k-2n$. Plugging $\omega = \omega_k$ and $\varphi$ into the eigenfunctions $v(x) = \sin(\omega x - \varphi)$, with the splitting $k=2\ell + r$ and $r \in \{0,1\}$, we obtain
\[
  v(x)
  = \sin \left( \omega_k x + \frac{\pi}{4} - \frac{\pi}{2} r + \pi(\ell - 2n) \right)
  = \pm \sin \left( \omega_k x + \frac{\pi}{4} - \frac{\pi}{2} r \right)
\]
where in the last step we have used that sine is $\pi$ periodic up to a sign, which is unimportant for the eigenfunctions. Thus, we obtain the eigenvalues and vectors of the lemma with $\varphi_k = - \frac{\pi}{4} + \frac{\pi}{2} r = - (-1)^k \frac{\pi}{4}$. With the same scaling argument and $\sin(- \omega x + \varphi) = - \sin(\omega x - \varphi)$ we can confine $\omega$ to non-negative numbers and likewise restrict to $k \ge 0$.

In order to show that $\phi_k$ is $\hs$ normalized, note that
\[
  \|\phi_k\|^2
  = \int_{-1}^1 \sin(\omega_k x - \varphi_k)^2 \, dx
  = \frac{1}{2} \int_{-1}^1 1 - \cos(2\omega_k x - 2\varphi_k) \, dx.
\]
Since $2 \varphi_k = \pm \frac{\pi}{2}$, we have $\cos(2\omega_k x - 2\varphi_k) = \pm \sin(2\omega_k x)$ and thus
\[
  \|\phi_k\|^2
  = \frac{1}{2} \int_{-1}^1 1 - \sin(2\omega_k x) \, dx
  = 1,
\]
where in the last step we have used that sine is an odd function so that its integral on the symmetric domain $[-1,1]$ vanishes.

\end{proof}

\subsection{Weight Distance}
\label{sec:shallow:weight-dist}

The following standard lemma from NTK theory bounds the distance of the weighs to their initial used in Assumption \ref{eq:convergence-general:weight-diff} of Theorem \ref{th:convergence-general}.

\begin{lemma} \label{lemma:shallow:weight-bound}
  For the shallow network \ref{eq:nn} and residual $\res(t) = f_{\theta(t)} - g$, we have
  \[
    \|\theta(t) - \theta(0)\|_\infty \le \sqrt{\frac{2}{m}} \int_0^t \|\res(\tau)\|_0 \, d\tau.
  \]
\end{lemma}

\begin{proof}

For each component $\theta_r$ of $\theta$ and unit basis vector $e_r$, we have
\begin{align*}
  |\theta_r(t) - \theta_r(0)|
  & = \left| \int_0^t \frac{d}{d\tau} \theta_r(\tau) \, d\tau \right|
  \\
  & \le \int_0^t \left| -\dualp{ e_r, \nabla L(\theta)} \right| \, d\tau
  = \int_0^t \left| \dualp{e_r, (Df_\theta)^* \res} \right| \, d\tau
  \\
  & = \int_0^t \left| \frac{1}{\sqrt{m}} \int_\dom a_r \sigma'(y-\theta_r) \res(y) \, dy \right| \, d\tau
  \\
  & = \int_0^t \left| \frac{1}{\sqrt{m}} \|a_r \sigma'(\cdot-\theta_r)\|_0 \|\res\|_0 \, dy \right| \, d\tau
  \\
  & \le \sqrt{\frac{2}{m}} \int_0^t \|\res(\tau)\|_0 \, d\tau,
\end{align*}
where in the second row we have used that $\theta(t)$ defined by the gradient flow \eqref{eq:shallow:gradient-flow} and that the gradient of the loss is given by the dual of the derivative applied to the residual \eqref{eq:gradient}. In the third row, we use the explicit formula \eqref{eq:shallow:derivatives} for the derivative of shallow networks and in the forth row we use that $a_r = \pm 1$ and $\|\sigma'(\cdot - \theta_r)\|_0 \le \sqrt{2}$.

\end{proof}

\subsection{Partial Derivatives Bounds}
\label{sec:shallow:derivative-bounds}

With the explicit eigenfunctions in Lemma \ref{lemma:shallow:eigenbasis} the partial derivatives
\begin{equation} \label{eq:shallow:partial-derivatives}
  \partial_r f_\theta
  = x \to \frac{a_r}{\sqrt{m}} \sigma'(x - \theta_r) \in \hs
\end{equation}
form \eqref{eq:shallow:derivatives} in the eigenbasis are given by
\begin{multline} \label{eq:shallow:derivative-eigenbasis}
  \dualp{\partial_r f_\theta, \phi_k}
  = \frac{a_r}{\sqrt{m}} \int_\dom \sigma'(x - \theta_r) \phi_k(x) \, dx
  \\
  = \frac{a_r}{\sqrt{m}} \int_{\theta_r}^1 \phi_k(x) \, dx
  = - \frac{a_r}{\omega_k \sqrt{m}} \left[ \cos(\omega_k - \varphi_k) - \cos(\omega_k \theta_r - \varphi_k) \right].
\end{multline}

\subsubsection{Upper Bounds}

\begin{lemma} \label{lemma:shallow:derivative-bounds}

  For the shallow network \eqref{eq:nn} and $s < \frac{1}{2}$, the partial derivatives $\partial_r f_\theta$ depend only on $\theta_r$ and we have $\|\partial_r f_\theta\|_s \le \frac{\mu}{\sqrt{m}}$ for some $\mu > 0$ independent of $m$.

\end{lemma}

\begin{proof}

The explicit formula \eqref{eq:shallow:derivative-eigenbasis} for the partial derivatives directly shows that $\partial_r f_\theta$ only depends on $\theta_r$. The same formula with $a_r = \pm 1$ and bounded cosine yields
\[
  m \|\partial_r f_\theta\|_s^2
  = \sum_{k=0}^\infty m \omega_k^{2s} \dualp{\partial_r f_\theta, \phi_k}^2
  \le \sum_{k=0}^\infty 4 \omega_k^{2s-2}
  =: \mu.
\]
Using that $\omega_k \sim k$, we have $\mu < \infty$ if and only if $2s-2 < -1$ or equivalently if $s < \frac{1}{2}$.

\end{proof}

\subsubsection{Simple Continuity Result}

\begin{lemma}

  For the shallow network \eqref{eq:nn} and $- \frac{1}{2} < s < \frac{1}{2}$ we have
  \[
    \|\partial_r f_\theta - \partial_r f_\ptheta\|_s \le c \frac{1}{\sqrt{m}} \|\theta-\ptheta\|_\infty^{\frac{1}{2} - s}
  \]
  for all $\theta,\ptheta \in \wdom$ and some constant $c > 0$ independent of $m$.

\end{lemma}

\begin{proof}

The explicit formula \eqref{eq:shallow:derivative-eigenbasis} for the partial derivatives implies
\begin{align*}
  \dualp{\partial_r f_\theta, \phi_k} - \dualp{\partial_r f_\ptheta, \phi_k}
  & = - \frac{a_r}{\omega_k \sqrt{m}} \left[ \cos(\omega_k \ptheta_r - \phi_k) - \cos(\omega_k \theta_r - \phi_k) \right].
  \\
  & \le \frac{a_r}{\omega_k \sqrt{m}} \min \{\omega_k |\theta_r - \ptheta_r|, 2\},
\end{align*}
where we have used that the cosine is $1$-Lipschitz if $|\theta_r - \ptheta|$ is sufficiently small and that the cosine is bounded if it is not. Thus, by $a_r^2 = 1$, with $\alpha = s-1$ and $-\frac{3}{2} < s-1 < \frac{1}{2}$, we have
\begin{align*}
  \|\partial_r f_\theta - \partial_r f_\ptheta\|_s^2
  & = \sum_{k=0}^\infty \omega_k^{2s} \dualp{\partial_r f_\theta - \partial_r f_\ptheta, \phi_k}^2
  \\
  & \lesssim \sum_{k=0}^\infty \omega_k^{2s} \frac{a_r^2}{\omega_k^2 m} \min \{\omega_k |\theta_r - \ptheta_r|, 2\}^2
  \\
  & \lesssim \frac{1}{m} \sum_{k=0}^\infty \omega_k^{2s-2} \min \{\omega_k |\theta_r - \ptheta_r|, 1\}^2
  \\
  & \lesssim \frac{1}{m} |\theta_r - \ptheta_r|^{1-2s},
\end{align*}
where the last inequality follows from  Lemma \ref{lemma:min-sum}. Taking the square root completes the proof.

\end{proof}

\subsubsection{Improved Continuity Result}
\label{sec:shallow:perturbation}

From the explicit formula \eqref{eq:shallow:partial-derivatives} for the partial derivatives, their difference is given by
\begin{equation*}
  \partial_r f_\theta  - \partial_r f_\ptheta
  = x \to \frac{a_r}{\sqrt{m}} [\sigma'(x - \theta_r) - \sigma'(x - \ptheta_r)],
\end{equation*}
supported on $x \in [\theta_r, \ptheta_r|$, which is fairly small in our application. By the random initialization, it is unlikely that the $\theta_r$ cluster at one point and therefore with high probability the differences $\partial_r f_\theta  - \partial_r f_\ptheta$ have disjoint support.  Hence, they behave like orthogonal functions in norms that localize, i.e. $\left\|\sum_{P \in \mathcal{P}} f_P \right\|_s^2 \lesssim \sum_{P \in \mathcal{P}} \|f_P\|_s^2$ for all function $f_P$ with bounded support and overlaps. At the end of Section \ref{sec:perturbations}, we have seen this leads to sharper perturbation bounds. Since we have defined the $\|\cdot\|_s$ norm by weighted $\ell_2$ norms on the eigenbasis coefficients, the localization property is not directly obvious. However, for two values of $s$, this locality can be easily seen. Indeed, for $s=0$ the $\|\cdot\|_s$ norm is identical to the $L_2(\dom)$ norm, which does localize. Moreover, sine $H$ is the inverse of a Laplacian with appropriate boundary conditions, the $\|\cdot\|_{-2}$ norm is the same as the $H^2(\dom)$ semi-norm, which also localizes. We extend the necessary localization to any $\hs^s$ norm by interpolation of function spaces to obtain the following result.

\begin{lemma} \label{lemma:shallow:perturbation-local}

  For the shallow network \eqref{eq:nn}, let the weights $\theta \in \wdom$, be i.i.d. uniformly distributed on $\wdom$ and assume that $0 \le s < \frac{1}{2}$. Then for any $h \ge 0$, with probability at least $1 - \frac{c}{h} e^{-2m \wdiff^2}$, we have
  \[
    \sup_{\|\nu\|_\infty \le 1} \left\| \sum_{r=1}^m (\partial_r f_\theta - \partial_r f_\ptheta) \nu_r \right\|_s \le c \sqrt{m} \wdiff^{1-s}
  \]
  for all $\ptheta \in \wdom$ with $\|\theta - \ptheta\|_\infty \le \wdiff$ and some constant $c > 0$ independent of $m$.

\end{lemma}

The proof requires some preparatory steps. First, we compute the space $\hs^2$.

\begin{lemma} \label{lemma:shallow:hs2}
  The Hilbert space $\hs^2$ of all functions with norm $\|\cdot\|_2 < \infty$ is equivalent to the Sobolev space
  \[
    \bar{H}^2(\dom) := \left\{ v \in H^2(\dom) \middle| \, v(-1) = 0, \, v'(1) = 0 \right\}.
  \]
  with identical norms
  \begin{align*}
    2 \|v\|_2 & = |v|_{H^2(\dom)}, &
    v & \in \bar{H}^2(\dom).
  \end{align*}

\end{lemma}

\begin{proof}

From Lemma \ref{lemma:ntk-inverse-laplace}, we know that $H: L_2(\dom) \to \bar{H}^s(\dom)$ is an isomorphism. Hence $v \in \bar{H}^2(\dom)$ if and only if $H^{-1} v \in L_2(\dom)$ if and only if $\|v\|_2 \sim \|H^{-1} v\|_0 < \infty$ is bounded, where the latter identity follows from
\[
  4 \| H^{-1} v\|_{L_2(\dom)}^2
  = \left\| \sum_{k=0}^\infty \omega_k^2 \dualp{v, \phi_k} \phi_k \right\|_{L_2(\dom)}^2
  = \sum_{k=0}^\infty \omega_k^4 \dualp{v, \phi_k}^2
  = \|v\|_2^2
\]
because by definition, $\phi_k$ are eigenfunctions of $H$ with eigenvalues $2 \omega_k^{-2}$. Finally, by Lemma \ref{lemma:ntk-inverse-laplace}, we have $H^{-1} v = v''$ and thus $2 |v|_{H^2(\dom)} = 2 \|v''\|_{L_2(\dom)} = 2  \|H^{-1} v\|_{L_2(\dom)} = \|v\|_2$.

\end{proof}

In order to find equivalent norms for the spaces $\hs^s$ with $s \not\in \{0,2\}$, we define the K-functional and interpolation spaces of two Banach spaces $X$ and $Y$, by
\begin{align*}
  K(f,t)
  & := \inf_{g \in Y} \|f - g\|_X + t \|g\|_Y,
  \\
  \|f\|_{(X,Y)_{\vartheta, q}}
  & := \left( \int_0^\infty \left[ t^{-\vartheta} K(f,t) \right]^q \, \frac{dt}{t} \right)^{1/q}
\end{align*}
for any $0< \vartheta < 1$ and $1 \le q \le \infty$.

\begin{corollary} \label{cor:shallow:interpolation-space}

  For any $0 < s < 2$, the spaces $\hs^s$ and $(L_2(\dom), \bar{H}^2(\dom))_{\frac{s}{2}, 2}$ are equivalent with equivalent norms.

\end{corollary}

\begin{proof}

We have $L_2(\dom) = \hs^0$ and by Lemma \ref{lemma:shallow:hs2} we have $\bar{H}^s(\dom) = \hs^2$. Thus, standard interpolation results, e.g. in \cite[Theorem 5.4.1]{BerghLofstrom1976}, imply
\[
  \|v\|_{(L_2(\dom), \bar{H}^2(\dom))_{\frac{s}{2}, 2}}
  \sim \sum_{k=0}^\infty 1^{1-\frac{s}{2}} \omega_k^{4 \frac{s}{2}} \dualp{v, \phi_k}^2
  = \|v\|_s^2.
\]

\end{proof}

Results e.g. in \cite{BechtelEgert2019} in combination with the reiteration theorem \cite{BerghLofstrom1976} indicate that the interpolation space $(L_2(\dom), \bar{H}^2(\dom))_{\frac{2}{s}, 2}$, $s \le \frac{1}{2}$ may be independent of the boundary conditions in $\bar{H}^2(\dom)$ and therefore a standard Besov space with suitable localization properties. Nonetheless, in the following we consider an elementary construction. In order to compute the interpolation norms, the following lemma provides an explicit estimate for the $K$-functional.

\begin{lemma} \label{lemma:shallow:K-functional}
  Let $I \subset \dom$ be an interval and $\chi_I$ be the corresponding characteristic function. Then, for all $t>0$ there exists a $g \in \bar{H}^2(\dom)$ such that
  \begin{align*}
    K(\chi_I, t) & \lesssim \min \{|I|^{1/2}, t^{1/4}\},
    \\
    \supp(g) & \subset \left\{ \begin{array}{ll}
      I + [t^{1/2}, t^{1/2}] & \text{if } t^{1/2} \le |I|, \\
      0 & \text{else}.
    \end{array} \right.
  \end{align*}

\end{lemma}

\begin{proof}

For $\epsilon = t^{1/2}$, let $\varphi_i$, $i = 1, \dots, P$ be a partition of unity with
\begin{align*}
  0 & \le \varphi_i \le 1, &
  \sum_{i=1}^P \varphi_i(x) & = 1, &
  |\supp (\varphi_i) | & \le \epsilon, &
  P \lesssim & \frac{1}{\epsilon}, &
  \|\varphi_i\|_{\bar{H}^2(\dom)} & \lesssim \epsilon^{-3/2}
\end{align*}
and a uniformly bounded number of overlapping supports. For the last inequality, note that partitions of unity are typically constructed by scaling $\varphi_i = \varphi\left(\frac{x}{\epsilon} - c_i \right)$. Then the second derivative in the norm contributes a scaling factor of $\epsilon^{-2}$ and rescaling in the $L_2(\dom)$ norm another factor of $\epsilon^{1/2}$, yielding the total factor of $\epsilon^{-3/2}$. With $\indexset(I) := \{i | \, I \cap \supp(\varphi_i) \ne \emptyset \text{ and } \partial \dom \cap \supp(\varphi_i) = \emptyset \}$ define
\[
  g(x) = \left\{ \begin{array}{ll}
    0 & \epsilon > |I| \\
    \sum_{i \in \indexset(I)} \varphi_i(x) & \text{else}.
  \end{array} \right.
\]
Note that by construction $g$ satisfies the boundary conditions of $\bar{H}(\dom)$ and has support contained in the extended interval $I + [-\epsilon, \epsilon]$.

Let us now estimate the $K$-functional of $\chi_I$ for the case that $\epsilon \le |I|$, i.e. the case where $g$ is non-zero. For convenience, define the domain $\dom_\epsilon = [-1+\epsilon, 1-\epsilon]$ of all points that are $\epsilon$ away from the boundary. Since $\varphi_i$ is a partition of unity, we have $g(x) = \chi_I$ for all $x \in \dom$, except for $\supp(g) \setminus \supp(\chi_I)$ and $\dom \setminus \dom_\epsilon$, where the difference between $g(x)$ and $\chi_I$ is at most $1$. Thus
\[
  \|\chi_I - g\|_{L_2(\dom)} \le \left( \int_{\supp(g) \setminus \supp(\chi_I) \cup \dom \setminus \dom_\epsilon} 1 \right)^{1/2}
  \lesssim \epsilon^{1/2}.
\]
Likewise, since $g(x) = \chi_I = 1$ on $I$ we have $g''(x) = 0$ on $I$ and thus
\[
  \|g\|_{\bar{H}^2(\dom)}
  \le \sum \|\varphi_i\|_{\bar{H}^2(\dom)}
  \lesssim \epsilon^{-3/2},
\]
with the sum taken over all $i$ for which $\supp(\varphi_i)$ intersects a boundary point of $I$. Thus, with $\epsilon = t^{1/2}$, we conclude that
\[
  K(\chi_I, t)
  = \inf_g \|\chi_I - g\|_{L_2(\dom)} + t \|g\|_{\bar{H}^2(\dom)}
  \lesssim \epsilon^{1/2} + t \epsilon^{-3/2}
  \lesssim t^{1/4}.
\]
Finally, for $|I|^{1/2} \le t^{1/4} = \epsilon^{1/2}$, we have defined $g=0$ and thus
\[
  K(\chi_I, t)
  \le \inf_g \|\chi_I\|_{L_2(\dom)}
  = |I|^{1/2}.
\]

\end{proof}

The following lemma provides guarantees that the weights do not cluster so that most differences $\partial_r f_\theta - \partial_r f_\ptheta$ have disjoint support and behave like orthogonal functions as we have seen in the introduction to this section.

\begin{lemma} \label{lemma:shallow:random-partition}

  Let $\theta_r$ be i.i.d. uniformly distributed random variables on $\dom$ and let $\mathcal{P}$ be a partition of $\dom$ so that for every $P \in \mathcal{P}$
  \begin{align*}
    |P| & \sim h, &
    \hat{P} & := P + [-h, h].
  \end{align*}
  Then, for any $m \ge 0$ and $\tau \ge 0$, with probability at least $1 - \frac{c}{h} e^{-2 m \tau^2}$
  \begin{align*}
    \frac{1}{m} \sum_{r=1}^m \chi_{\hat{P}} (\theta_r)
    & \lesssim h + \tau, &
    \text{for all }P & \in \mathcal{P}
  \end{align*}
  for some $c \ge 0$.

\end{lemma}

\begin{proof}

The bounded random variable $0 \le \chi_{\hat{P}}(\theta) \le 1$ has expectation
\[
  \E{\chi_{\hat{P}}(\theta)}
  \le \frac{1}{2} (|P| + 2h)
  \lesssim |P| + h
  \lesssim h,
\]
where the extra factor $\frac{1}{2}$ comes from the normalization of the uniform distribution on the domain $\dom = [-1, 1]$. Thus, by Hoeffding's inequality
\[
  \pr{ \frac{1}{m} \sum_{r=1}^m \chi_{\hat{P}}(\theta_r) - \E{\chi_{\hat{P}}(\theta)} \ge \tau }
  \le
  e^{-2 m \tau^2}.
\]
The lemma follows from a union bound on the $\sim \frac{1}{h}$ elements of the partition $\mathcal{P}$.

\end{proof}

\begin{proof}[Proof of Lemma \ref{lemma:shallow:perturbation-local}]

For any vector $v \in \real^m$ with $\|v\|_\infty \le 1$, set
\begin{align*}
  f_r & = \partial_r f_\theta - \partial_r f_\ptheta, &
  F & = \sum_{r=1}^m f_r v_r.
\end{align*}
We find a bound for $\|F\|_s$ based on a variation of the norm-localization property \eqref{eq:norm-local}. Since localization is not immediately clear for the eigenfunction based definition of the $\hs^s$ norm, we use the equivalent characterization via interpolation spaces from Corollary \ref{cor:shallow:interpolation-space}. Thus, we first estimate the K-functional. By \eqref{eq:shallow:partial-derivatives}, the partial derivatives are given by
\[
  v_r = \partial_r f_\theta - \partial_r f_\ptheta
  = \frac{a_r}{\sqrt{m}}[ \sigma'(x-\theta_r) - \sigma'(x - \ptheta_r)]
  = (-1)^{s_r} \frac{a_r}{\sqrt{m}} \chi_{[\theta_r, \ptheta_r]}(x)
\]
for some $s_r \in \{0,1\}$, depending the order $\theta_r \lessgtr \ptheta_r$ and the sign of $a_r$. Let $g_r$ be the smooth approximation of $\chi_{[\theta_r, \ptheta_r]}$ constructed in Lemma \ref{lemma:shallow:K-functional} with the extra scaling factor $(-1)^{s_r} \frac{a_r}{\sqrt{m}}$, so that
\begin{align*}
  \|v_r - g_r\|_{L_2(\dom)} + t\|g_r\|_{\hat{H}^2(\dom)} & \lesssim \frac{1}{\sqrt{m}} \min\{h^{1/2}, t^{1/4}\}
  \\
  \supp{g_r} & \subset [\theta_r - h, \ptheta_r + h],
\end{align*}
where we have used that either $t \le h$ so that  $[\theta_r - t, \ptheta_r + t] \subset [\theta_r - h, \ptheta_r + h]$ or $g_r$ is identically zero. With  $G := \sum_{r=1}^m g_r v_r$, the K-functional is bounded by $K(F,t) \le \|F-G\|_{L_2(\dom)} + t \|G\|_{H^s(\dom)}$. In order to obtain a sharp bound, we exploit that both involved norms localize in the sense of \eqref{eq:norm-local} and thus, we estimate next how many supports of $f_r - g_r$ may overlap with some uniform reference partition $\mathcal{P}$ of $\dom$ with step size $2h$. By Lemma \ref{lemma:shallow:random-partition} with $\tau=h$, with probability $1 - \frac{c}{h} e^{-2 m h^2}$, each $2h$-extended interval $\hat{P}$ of the partition contains at most $\lesssim hm $ weights $\theta_r$. This implies that that there are at most $\lesssim hm$ functions $f_r - g_r$ with support overlapping with $P$. Indeed, for an overlapping support it is necessary that $\theta_r \in \hat{P}$ because by the definition of $g_r$ and $|\theta_r - \ptheta_r| \le h$, we have $\supp(f_r - g_r) \subset [\theta_r - 2h, \theta_r + 2h]$. Thus, we have
\begin{align*}
  K(F,t)^2
  & \lesssim \|F-G\|_{L_2(\dom)}^2 + t^2 \|G\|_{H^s(\dom)}^2
  \\
  & = \sum_{P \in \mathcal{P}} \left\{ \left\| \sum_{r=1}^m (f_r - g_r) v_r \right\|_{L_2(P)}^2 + t \left\| \sum_{r=1}^m g_r v_r \right\|_{H^s(P)}^2 \right\}
  \\
  & \le \sum_{P \in \mathcal{P}} \left\{ \left(\sum_{\theta_r \in \hat{P}} \| (f_r - g_r) \|_{L_2(P)} \right)^2 + \left( \sum_{\theta_r \in \hat{P}} t \| g_r \|_{H^s(P)} \right)^2 \right\}
  \\
  & \le \sum_{P \in \mathcal{P}} \left\{ \left(hm \frac{1}{\sqrt{m}} \min\{h^{1/2}, t^{1/4} \} \right)^2 + \left( hm \frac{1}{\sqrt{m}} \min\{h^{1/2}, t^{1/4} \} \right)^2 \right\}
  \\
  & \lesssim \sum_{P \in \mathcal{P}} h^2 m \min\{h^{1/2}, t^{1/4} \}^2
  \\
  & \lesssim h m \min\{h, t^{1/2} \}.
\end{align*}
Since by Corollary \ref{cor:shallow:interpolation-space} the $\hs^s$ norm for $s>0$ is equivalent to an interpolation norm, we obtain
\begin{align*}
  \|F\|_s
  & \sim \|F\|_{(L_2(\dom), \bar{H}^2(\dom))_{fr{2}{s}, 2}}
  \\
  & = \left( \int_0^\infty t^{-s} K(F,t)^2 \, \frac{dt}{t} \right)^{1/2}
  \\
  & \le \sqrt{hm} \left( \int_0^\infty t^{-s-1} \min\{h^2, t\}^{1/2} \, dt \right)^{1/2}
  \\
  & \lesssim \sqrt{hm} \left((h^2)^{-s - 1 + \frac{1}{2} + 1} \right)^{1/2}
  \\
  & = \sqrt{m} h^{1-s},
\end{align*}
where we have computed the integral with Lemma \ref{lemma:min-int} with $\alpha = -s-1$ and $\beta = \frac{1}{2}$. The conditions $\alpha < -1$ and $\alpha + \beta > -1$ satisfied because $0 < s < \frac{1}{2}$.

It remains to consider the special case $s=0$. In the limit $t \to \infty$, we have $G = 0$ and thus the same reasoning as above implies
\[
  \|F\|_{L_2(\dom)}^2
  = \|F-G\|_{L_2(\dom)}^2 + t^2 \|G\|_{H^s(\dom)}^2
  \lesssim h m \min\{h, t^{1/2} \}
  = m h^2.
\]

\end{proof}

\subsection{Proof of the Theorem}
\label{eq:shallow:proof-combine}

\begin{proof}[Proof of Theorem \ref{th:convergence-bias1d}]

The Theorem follows from Theorem \ref{th:convergence-general}. We first verify all assumptions.
\begin{enumerate}

  \item \emph{Weight Distance:} The bound \eqref{eq:convergence-general:weight-diff} for $\|\theta(t)-\theta(0)\|_\infty$ follows directly from Lemma \ref{lemma:shallow:weight-bound}.
  
  \item \emph{Interpolation Inequality}: The interpolation inequality \eqref{eq:interpolation-inequality} is well known for sequence spaces, see e.g. \cite{BerghLofstrom1976}. For completeness, it is included in Lemma \ref{lemma:interpolation-inequality}.

  \item \emph{Coercivity:} Expanding $v$ in the orthonormal eigenbasis $\phi_k$ of $H$ with eigenvalues $2 \omega_k^{-2}$ (Lemma \ref{lemma:shallow:eigenbasis}), we have
  \[
    (Hv)_k
    = \dualp{\phi_k, Hv}
    = \sum_{k=0}^\infty v_\ell \dualp{\phi_k, H \phi_\ell}
    = \sum_{k=0}^\infty v_\ell 2 \omega_\ell^{-2} \dualp{\phi_k, \phi_\ell}
    = 2 \omega_k^{-2} v_k
  \]
  and thus
  \[
    \dualp{v, H v}_S
    = \sum_{k=0}^\infty v_k \omega_k^{2S} (Hv)_k
    = \sum_{k=0}^\infty 2 v_k \omega_k^{2(S-1)} v_k
    = 2 \|v\|_{S-1}^2.
  \]
  This yields the coercivity assumptions \eqref{eq:convergence-general:coercive}.

  \item \emph{Initial concentration:} By the explicit formula \eqref{eq:shallow:derivatives} for the derivative $Df_\theta$, the partial derivatives are given by $\partial_r f_\theta e_r = \frac{1}{\sqrt{m}} a_r \sigma'(x - \theta_r)$, which are clearly independent of $\theta_{\bar{r}}$ for any $\bar{r} \ne r$. Then, by Lemma \ref{lemma:shallow:derivative-bounds} for $s<\frac{1}{2}$ the derivatives are bounded by $\|\partial_r f_\theta\|_s \le \frac{\mu}{\sqrt{m}}$ so that assumption \eqref{eq:convergence-general:initial} of Theorem \ref{th:convergence-general} is satisfied.

  \item \emph{Perturbation:} By Lemma \ref{lemma:shallow:perturbation-local}, for $0 \le s < \frac{1}{2}$ and any $\wdiff \ge 0$, with probability at least $1 - \frac{c}{\wdiff} e^{-2m \wdiff^2}$, we have
  \[
    \sup_{\|\nu\|_\infty \le 1} \left\| \sum_{r=1}^m (\partial_r f_\theta - \partial_r f_\ptheta) \nu_r \right\|_s \le \sqrt{m} c \wdiff^{1-s}
  \]
  for all $\ptheta \in \wdom$ with $\|\theta - \ptheta\|_\infty \le \wdiff$ and some constant $c > 0$ independent of $m$.This directly shows assumption \eqref{eq:convergence-general:perturbation} of Theorem \ref{th:convergence-general} with $\alpha = 1-s$ and $p_0(m,\wdiff) = \frac{c}{\wdiff} e^{-2m \wdiff^2}$.

\end{enumerate}
Thus, the result follows from Theorem \ref{th:convergence-general} with $\alpha = 1-s$ and $p_0(m,\wdiff) = \frac{c}{\wdiff} e^{-2m \wdiff^2}$.

\end{proof}

\section{Auxiliary Results}
\label{sec:auxiliary}

\subsection{Matrix Bernstein Inequalities}

This section contains matrix concentration inequalities \cite{Tropp2015,Vershynin2018} used to show concentration of the NTK at the initial value. They are all based on dimension free matrix Bernstein inequalities \cite{HsuKakadeZhang2012,Tropp2015,Minsker2017}, from which we use the former. It is less sharp but does not require lower bounds of operator norms in its assumptions.

\begin{theorem}[{\cite[Theorem 3.3, special case]{HsuKakadeZhang2012}}]
  \label{th:matrix-bernstein}
  Let $\xi_i$, $i=1,\dots,n$ be independent random variables and $X_i = X_i(\xi_i)$ be symmetric matrices of same size. Assume there are $b > 0$, $\sigma > 0$ and $k > 0$ such that for all $i=1,\dots,n$
  \begin{align*}
    \E{X_i} & = 0, &
    \trace \left( \frac{1}{n} \sum_{i=1}^n \E{X_i^2} \right) & \le \sigma^2 k,
    \\
    \lmax(X_i) & \le b, &
    \lmax \left( \frac{1}{n} \sum_{i=1}^n \E{X_i^2} \right) & \le \sigma^2,
  \end{align*}
  almost surely. Then, for any $t > 0$,
  \[
    \pr{
      \lmax \left( \frac{1}{n} \sum_{i=1}^n X_i \right)
      >
      \sqrt{ \frac{2 \sigma^2 t}{n}} + \frac{bt}{3n}
    }
    \le
    k \cdot t \left( e^t - t - 1 \right)^{-1}.
  \]

\end{theorem}

We use the following standard symmetrization \cite{Tropp2015} to extend the result to non-symmetric matrices.

\begin{lemma} \label{lemma:symmetrize}
  Let $U$ and $V$ be two Hilbert spaces and $L: U \to V$ be a linear operator. Then
  \[
    \|L\|_{U,V} = \left\| \begin{bmatrix}
        & L^* \\
      L &
    \end{bmatrix} \right\|_{V \times U, V \times U}.
  \]
\end{lemma}

\begin{proof}

We have
\[
  \dualp{
    \begin{pmatrix} u \\ v \end{pmatrix}, \,
    \begin{bmatrix} & L^* \\ L & \end{bmatrix}
    \begin{pmatrix} w \\ x \end{pmatrix}
  }
  = \dualp{Lu, x} + \dualp{v,Lw}
\]
and
\[
  \|u\|_U \|x\|_V + \|v\|_V \|w\|_U
  \le \left(\|u\|_U^2 + \|v\|_V\right)^{1/2}\left(\|w\|_U^2 + \|x\|_V\right)^{1/2}
  = \|[u,v]\|_{U,V} \|[w,x]\|_{U,V}
\]
by the Cauchy Schwartz inequality applied to the two summands on the left hand side. Thus
\begin{align*}
  \left\| \begin{bmatrix}
      & L^* \\
    L &
  \end{bmatrix} \right\|_{V \times U, V \times U}
  & = \sup_{u,v \ne 0} \sup_{w,x \ne 0} \frac{\dualp{Lu, x} + \dualp{v,Lw}}{\|[u,v]\|_{U,V} \|[w,x]\|_{U,V}}
  \\
  & \le \sup_{u,v \ne 0} \sup_{w,x \ne 0} \frac{\|u\|_U \|x\|_V + \|v\|_V \|w\|_U}{\|[u,v]\|_{U,V} \|[w,x]\|_{U,V}} \|L\|_{U,V}
  \\
  & \le \|L\|_{U,V}.
\end{align*}
Conversely, inserting the zero vector $0$, we have
\begin{align*}
  \|L\|_{U,V}
  & = \sup_{v \ne 0}\sup_{w \ne 0} \frac{\dualp{v, L w}}{\|v\|_V \|w\|_U}
  \\
  & = \sup_{v \ne 0}\sup_{w \ne 0} \frac{\dualp{0, L 0} + \dualp{v, L w}}{\|[0,v]\|_{U,V} \|[w,0]\|_{U,V}}
  \\
  & \le \sup_{u,v \ne 0}\sup_{w,x \ne 0} \frac{\dualp{u, L x} + \dualp{v, L w}}{\|[u,v]\|_{U,V} \|[w,x]\|_{U,V}}
  \\
  & = \left\| \begin{bmatrix}
          & L^* \\
        L &
      \end{bmatrix} \right\|_{V \times U, V \times U}
\end{align*}

\end{proof}

The following corollary generalizes the dimension free matrix Bernstein inequality to non-zero means and non-symmetric $X_i$ by the symmetrization above. In addition, it makes some crude simplifications of several assumptions, which are sufficient for our purposes.

\begin{corollary}
  \label{cor:matrix-concentration}
  Let $\xi_i$, $i=1,\dots,n$ be independent random variables, $U$, $V$ Hilbert spaces and $X_i = X_i(\xi_i)$ Bochner integrable bounded operators from $U$ to $V$ with induced operator norm $\|\cdot\|$. Assume there are $\beta > 0$ and $k > 0$ such that for all $i=1,\dots,n$ pointwise
  \begin{align}
    \label{eq:matrix-concentration-norm}
    \|X_i\| & \le \beta
    \\
    \label{eq:matrix-concentration-trace}
    \trace(X_i^* X_i) & \le k \beta^2
    \\
    \label{eq:matrix-concentration-rank}
    \operatorname{rank}(X_i) & \le r
  \end{align}
  almost surely. Then, for any $t > 0$,
  \[
    \pr{
      \left\| \frac{1}{n} \sum_{i=1}^n X_i - \E{X_i} \right\|
      >
      \sqrt{ \frac{8 \beta^2 t}{n}} + \frac{2 \beta t}{3n}
    }
    \le
    2 k t \left( e^t - t - 1 \right)^{-1}.
  \]

\end{corollary}

\begin{proof}[Proof: Finite Dimensional Case]

We first prove the result for finite dimensional $U$ and $V$. We symmetrize all $X_i$ by Lemma \ref{lemma:symmetrize}, which yields the symmetric matrices
\begin{align*}
  S_i & := \begin{pmatrix} & X_i^* \\ X_i & \end{pmatrix}, &
  \bar{S}_i(\xi_i) & := S_i - \E{S_i}.
\end{align*}
with unchanged operator norms. Then, the result follows from Theorem \ref{th:matrix-bernstein} applied to $\bar{S}_i$ after we have verified all assumptions.

\begin{enumerate}

  \item We have
  \[
    \E{\bar{S}_i} = 0.
  \]

  \item We estimate the eigenvalues by
  \begin{equation*}
    \|\bar{S}_i\|
    \le \|S_i\| + \|\E{S_i}\|
    \overset{(*)}{\le} \|S_i\| + \E{\|S_i\|}
    \overset{(**)}{=} \|X_i\| + \E{\|X_i\|}
    \le 2 \beta,
  \end{equation*}
  where in $(*)$ we have used Jensen's inequality, in $(**)$ we have used that the symmetrization Lemma \ref{lemma:symmetrize} leaves norms unchanged and the last inequality follows from assumption \eqref{eq:matrix-concentration-norm}.

  \item Since $\bar{S}_i$ is symmetric, we have $\|\bar{S}_i^2\| = \|\bar{S}_i\|^2$. Therefore, with $\|\bar{S}\| \le 2 \beta$ from the previous item, we have
   \[
    \left\| \E{\bar{S}_i^2} \right\|
    \le \E{\left\| \bar{S}_i^2 \right\|}
    = \E{\left\| \bar{S}_i \right\|^2}
    = \E{4\beta^2}
    = 4 \beta^2,
  \]
  which directly implies
  \begin{align*}
    \left\| \frac{1}{n} \sum_{i=1}^n \E{\bar{S}_i^2} \right\|
    \le \frac{1}{n} \sum_{i=1}^n \left\| \E{\bar{S}_i^2} \right\|
    \le 4 \beta^2.
  \end{align*}

  \item For any symmetric matrix $M$, the function $M \to \trace(M^2)^{1/2} = \|M\|_{HS}$ is the Hilbert-Schmidt norm and therefore convex. Thus, by Jensen's inequality we have $\trace(\E{M}^2)^{1/2} \le \E{\trace(M^2)^{1/2}}$. Squaring this equation and using that the trace and expectation commute, we obtain
  \begin{multline*}
    \trace \left( \E{\bar{S_i}^2} \right)
    = \trace \left( \E{ \left(S_i - \E{S_i} \right)^2} \right)
    = \trace \left( \E{S_i^2} \right) - \trace \left( \E{S_i}^2 \right)
    \\
    \le \E{\trace \left( S_i^2 \right)}
    + \E{\trace \left( S_i^2 \right)^{1/2}}^2
    \le 2 \E{\trace \left( S_i^2 \right)},
  \end{multline*}
  where the last inequality follows from Cauchy-Schwarz. Since
  \[
    S_i^2 = \begin{bmatrix} X_i^* X_i & \\ & X_i X_i^* \end{bmatrix}
  \]
  with assumption \eqref{eq:matrix-concentration-trace}, we have
  \[
    \trace(S_i^2)
    = \trace(X_i^* X_i) + \trace(X_i X_i^*)
    = 2 \trace(X_i^* X_i)
    \le 2 k \beta^2.
  \]
  Plugging into the trace estimate yields
  \[
    \trace \left( \E{\bar{S_i}^2} \right)
    \le 4 k \beta^2
  \]
  and therefore
  \[
    \trace \left( \frac{1}{n} \sum_{i=1}^n \E{\bar{S}_i^2} \right)
    \le 4 k \beta^2.
  \]

\end{enumerate}
Since the finite dimensional Hilbert space $U \times V$ is isometrically isomorphic to $\ell_2(\real^d)$ for some $d$, without loss of generality, we assume that $U \times V = \ell_2(\real^d)$ in the following. Then, all assumptions of Theorem \ref{th:matrix-bernstein} are satisfied and with $b = 2\beta$, $\sigma^2 = 4 \beta^2$ and $k = k$ so that
\[
  \pr{
    \lmax \left( \frac{1}{n} \sum_{i=1}^n \bar{S}_i \right)
    >
    \sqrt{ \frac{8 \beta^2 t}{n}} + \frac{2 \beta t}{3n}
  }
  \le k t \left( e^t - t - 1 \right)^{-1}.
\]
Applying the same argument to $-\bar{S}_i$ instead of $\bar{S}_i$, we obtain the same estimate for $\lmax(-\cdot)$, so that we can replace $\lmax(\cdot)$ with the norm $\|\cdot\|$.
This completes the proof because by Lemma \ref{lemma:symmetrize} with $\bar{X}_i = X_i - \E{X_i}$ we have
\[
  \left\| \frac{1}{n} \sum_{i=1}^n \bar{S}_i \right\|
  = \left\| \begin{bmatrix}
      & \frac{1}{n} \sum_{i=1}^n \bar{X}_i^* \\
      \frac{1}{n} \sum_{i=1}^n \bar{X}_i  &
    \end{bmatrix} \right\|
  = \left\| \frac{1}{n} \sum_{i=1}^n \bar{X}_i \right\|.
\]

\end{proof}

\begin{proof}[Proof: Infinite Dimensional Case]

We have already proven the result in finite dimensions. Now, we extend it to infinite dimensions by a perturbation argument. Since $X_i(\xi_i)$ is Bochner integrable, we can approximate it with a simple function
\begin{equation*}
  X_{i\epsilon} = \sum_{k=1}^K \chi_{A_{ik}} (\xi_i) Y_{ik}
\end{equation*}
for some disjoint measurable sets $A_{ik}$ and operators $Y_{ik}$. By the argument in \cite[Lemma A1]{Kato1973}, without loss of generality, we can replace $Y_{ik}$ by some samples $X_i(\xi_{ik})$ for some $\xi_{ik}$ to obtain a Riemann-type sum with arbitrary error $\epsilon$:
\begin{align*}
  X_{i\epsilon} & = \sum_{k=1}^K \chi_{A_k} (\xi_i) X_i(\xi_{ik}), &
  \int \|X_i(\xi_i) - X_{i\epsilon}(\xi_i)\| \, dp(\xi_i) & \le \epsilon.
\end{align*}
Since all $X_{ik} = X_i(\xi_{ik})$ have rank bounded by $r$, the subspaces
\begin{align*}
  \bar{U} & = \bigoplus_{i,k} \operatorname{Ker}(X_{ik})^\perp \subset U, &
  \bar{V} & = \bigoplus_{i,k} \operatorname{Ran} X_{ik} \subset V
\end{align*}
are finite dimensional and
\begin{align*}
  \operatorname{Ker}(X_{i\epsilon}(\xi))^\perp & \subset \bar{U}, &
  \operatorname{Ran} X_{i\epsilon}(\xi) & \subset \bar{V}
\end{align*}
for all $i$ and $\xi$. Hence, without loss of generality, we can consider the random variables on the finite dimensional Hilbert spaces $\bar{U}$ and $\bar{V}$. Since we have chosen $Y_{ik} = X_{i}(\xi_k)$ as samples in the definition of the simple functions, $X_{i\epsilon}$ satisfies the pointwise assumptions \eqref{eq:matrix-concentration-norm}, \eqref{eq:matrix-concentration-trace} and \eqref{eq:matrix-concentration-rank}. Thus, with the abbreviations
\begin{align*}
  Z & := \frac{1}{n} \sum_{i=1}^n X_i - \E{X_i}, &
  Z_\epsilon & := \frac{1}{n} \sum_{i=1}^n X_{i\epsilon} - \E{X_{i\epsilon}}.
\end{align*}
by the finite dimensional version of the corollary, we have
\begin{align*}
  \pr{
    \|Z_\epsilon\| > b(\tau)
  }
  & \le p(\tau), &
  b(\tau) & := \sqrt{ \frac{8 \beta^2 \tau}{n}} + \frac{2 \beta \tau}{3n}, &
  p(\tau) & := 2 k \tau \left( e^\tau - \tau - 1 \right)^{-1}
\end{align*}
for any $\tau > 0$. We now pass to the limit $\delta, \epsilon \to 0$. To this end, note that
\[
  \E{\|Z - Z_\epsilon\|}
  \le \frac{1}{n} \sum_{i=1}^n \E{\|X_{i\epsilon} - X_i\|} + \|\E{X_{i\epsilon} - X_i}\|
  \le 2 \epsilon
\]
so that by Markov's inequality
\begin{equation*}
  \pr{\|Z - Z_\epsilon\| \ge \delta} \le \frac{2\epsilon}{\delta}.
\end{equation*}
Thus, choosing $\tau$ so that $b(\tau) = b(t) + \delta$, we obtain
\begin{align*}
  \pr{\|Z\| > b(t)}
  & \le \pr{\|Z\| > b(t)\text{ and } \|Z-Z_\epsilon\| \le \delta}
  + \pr{ \|Z - Z_\epsilon\| \ge \delta}
  \\
  & \le \pr{\|Z_\epsilon\| > b(t) - \delta}
  + \frac{2\epsilon}{\delta}
  \\
  & \le \pr{\|Z_\epsilon\| > b(\tau)}
  + \frac{2\epsilon}{\delta}
  \\
  & \le p(\tau) + \frac{2\epsilon}{\delta}
\end{align*}
and the result follows from first choosing an arbitrarily small $\delta$ so that $p(\tau) \to p(t)$ and then a even smaller $\epsilon$ so that $\frac{2\epsilon}{\delta} \to 0$.

\end{proof}

In our application, we only consider rank one matrices for which the last corollary can be further simplified.

\begin{corollary}
  \label{cor:matrix-concentration-rank-one}
  Let $\xi_i$, $i=1,\dots,n$ be independent random variables, $U$, $V$ Hilbert spaces and $X_i = X_i(\xi_i) = v_i(\xi_i) u_i(\xi_i)^* = v_i u_i^*$ be Bochner integrable rank one operators with $v_i \in V$ and $u_i^* \in U^*$. Assume there are $\mu > 0$ and $\nu > 0$ such that for all $i=1,\dots,n$
  \begin{align*}
    \|u\|_U & \le \mu, &
    \|v\|_V & \le \nu, &
  \end{align*}
  almost surely. Then, for any $t > 0$,
  \[
    \pr{
      \left\| \frac{1}{n} \sum_{i=1}^n X_i - \E{X_i} \right\|
      >
      \sqrt{ \frac{8 \mu^2 \nu^2 t}{n}} + \frac{2 \mu \nu t}{3n}
    }
    \le
    2 t \left( e^t - t - 1 \right)^{-1}.
  \]

\end{corollary}

\begin{proof}

We have
\begin{align*}
  \|X_i\| & = \|v_i\|_V \|u_i^*\|_{U^*} \le \mu \nu, \\
  \trace(X_i^* X_i) & = \trace(u_i v_i^* v_i u_i^*) = \|u_i\|^2 \|v_i\|^2 \le \mu^2 \nu^2.
\end{align*}
Thus, application of Corollary \ref{cor:matrix-concentration} with $\beta = \mu \nu$ and $k = 1$ yields the result.

\end{proof}

\subsection{Sums and Integrals}

\begin{lemma} \label{lemma:min-int}
  Let $\alpha < -1$ and $\alpha + \beta > -1$. Then for every $x > 0$
  \[
    \int_0^\infty t^\alpha \min\{x, t\}^\beta \, dt
    = - \frac{\beta}{\alpha(\alpha + \beta)} x^{\alpha + \beta + 1}.
  \]

\end{lemma}

\begin{proof}

We have
\begin{align*}
  \int_0^\infty t^\alpha \min\{x, t\}^\beta \, dt
  & = \int_0^x t^{\alpha + \beta} \, dt
     +\int_x^\infty t^\alpha x^\beta \, dt
  \\
  & = \frac{1}{\alpha + \beta} x^{\alpha + \beta + 1}
    - \frac{1}{\alpha} x^{\alpha+1} x^\beta
  \\
  & = - \frac{\beta}{\alpha(\alpha + \beta)} x^{\alpha + \beta + 1},
\end{align*}
where by assumption $\alpha + 1 < 0$ and $\alpha + \beta + 1 > 0$, so that the boundary terms for $x \to 0$ and $t \to \infty$ vanish and all integrals are well defined.

\end{proof}

\begin{lemma} \label{lemma:min-sum}

  Let $\omega_k = \frac{\pi}{4} + \frac{\pi}{2} k$ and $-\frac{3}{2} < \alpha < -\frac{1}{2}$. Then for every $x > 0$
  \[
    \sum_{k=0}^\infty \omega_k^{2\alpha} \min\{\omega_k x, 1\}^2
    \le - 2 \left(\frac{\pi}{2}\right)^{2\alpha+2} \frac{1}{\alpha (2\alpha+2)} x^{-2\alpha - 1}.
  \]

\end{lemma}

\begin{proof}

We have

\begin{align*}
  \sum_{k=0}^\infty \omega_k^{2\alpha} \min\{\omega_k x, 1\}^2
  & \le \left(\frac{\pi}{2}\right)^{2\alpha+2} \sum_{k=0}^\infty (k+1)^{2\alpha} \min\{(k+1) x, 1\}^2
  \\
  & =   \left(\frac{\pi}{2}\right)^{2\alpha+2} \sum_{k=1}^\infty k^{2\alpha} x^2 \min\left\{k, \frac{1}{x}\right\}^2
  \\
  & \le 2 \left(\frac{\pi}{2}\right)^{2\alpha+2} x^2 \int_0^\infty t^{2\alpha} \min\left\{t, \frac{1}{x}\right\}^2 \, dt
  \\
  & \le - 2 \left(\frac{\pi}{2}\right)^{2\alpha+2} \frac{1}{\alpha (2\alpha+2)} x^2 \left( \frac{1}{x} \right)^{2\alpha + 3},
\end{align*}
where in the first line we have used that $\omega_k \le \frac{\pi}{2}(k+1)$ and in the second an index shift. In the third, we bound the sum by an integral with an extra factor of $2$ because the integrand is monotonically increasing and decreasing in two respective sub-intervals. In the last line, we have invoked Lemma \ref{lemma:min-int}, using that $2\alpha < -1$ and $2\alpha + 2 > -1$ by assumption.

\end{proof}

\subsection{Differential Inequalities}

\begin{lemma} \label{lemma:ode-system-bounds}
  Assume $a, b, c, \rho > 0$ and that $x$, $y$ satisfy the differential inequality
  \begin{align}
    \label{eq:shallow:ode-x}
    x' & \le - a x^{1+\rho} y^{-\rho} + b x \\
    \label{eq:shallow:ode-y}
    y' & \le c \sqrt{x y}.
  \end{align}
  Then with
  \begin{equation} \label{eq:ode-system-bounds:smoothness}
    \gamma
    := 2 \left[ y(0)^{\frac{1}{2}} + 2^{\frac{1}{2\rho}} \frac{c}{b} x(0)^{\frac{1}{2}} \right]^2
  \end{equation}
  we have
  \begin{align*}
    x(t) & \le \left(\frac{b}{a} \gamma^\rho + x(0)^\rho e^{-b \rho t} \right)^{\frac{1}{\rho}}, &
    y(t) & \le \gamma
  \end{align*}
  for all $t$ for which the second summand in the $x$ bound dominates, i.e.
  \begin{equation} \label{eq:ode-system-bounds:time}
    \frac{b}{a} \gamma^\rho \le x(0)^\rho e^{-b \rho t}.
  \end{equation}
\end{lemma}

\begin{proof}

For any fixed $y$, the function $x$ is bounded by the solution $z$ of the equality case
\begin{align*}
  z' & = - a z^{1+\rho} y^{-\rho} + b z, &
  z(0) & = x(0)
\end{align*}
of the first equation \eqref{eq:shallow:ode-x}. This is a Bernoulli differential equation, with solution
\[
  x(t)
  \le z(t)
  = \left[ e^{-b\rho t} \left( a \rho \int_0^t e^{b \rho \tau} y(\tau)^{-\rho} \, d\tau + x(0)^{-\rho} \right) \right]^{-\frac{1}{\rho}}.
\]
We can use this estimate to eliminate $x$ from the differential inequality \eqref{eq:shallow:ode-y} for $y'$, but in order to obtain a more manageable ODE, we simplify $z$ first. To this end, let $T$ be the final time in the statement of the lemma for which \eqref{eq:ode-system-bounds:time} holds and assume for the time being that $\bgamma := \sup_{0 \le t \le T} y(t)$ is finite. Then
\begin{align*}
  z(t)^{\rho}
  & \le e^{b\rho t} \left( a \rho \int_0^t e^{b \rho \tau} \bgamma^{-\rho} \, d\tau + x(0)^{-\rho} \right)^{-1}
  \\
  & = e^{b\rho t} \left( \frac{a}{b} \left( e^{b \rho t} - 1 \right) \bgamma^{-\rho} + x(0)^{-\rho} \right)^{-1}
  \\
  & = \left( \frac{a}{b} \bgamma^{-\rho} - \left(\frac{a}{b} \bgamma^{-\rho} -  x(0)^{-\rho}\right) e^{-b \rho t} \right)^{-1}
  \\
  & = \underbrace{\frac{b}{a} \bgamma^\rho}_{=:A} \left( 1 - \underbrace{\left(1 -  \frac{b}{a} \left(\frac{x(0)}{\bgamma}\right)^{-\rho}\right) e^{-b \rho t}}_{=: B(t)} \right)^{-1}
\end{align*}
and thus
\begin{align*}
  z(t)^\rho
  \le \frac{A}{1 - B(t)}
  = \frac{A[1 - B(t)]}{1 - B(t)} + \frac{A B(t)}{1 - B(t)}
  = A + \frac{A}{1 - B(t)} B(t).
\end{align*}
Since $B(t) \ge 0$ by the time restriction \eqref{eq:ode-system-bounds:time} evaluated at $t=0$, the function $A / (1-B(t))$ is monotonically decreasing and thus with $A / (1-B(0)) = x(0)^\rho$, we have
\[
  z(t)^\rho
  \le A + \frac{A}{1 - B(0)} B(t)
  = A + x(0)^\rho B(t)
  \le A + x(0)^\rho e^{- b \rho t},
\]
which shows the bound for $x(t)$ in the lemma. Thus, it remains to ensure that $y(t) \le \bgamma \le \gamma$ for all times $t<T$ for which the second term dominates \eqref{eq:ode-system-bounds:time} or equivalently $A \le x(0)^\rho e^{-b \rho t}$. This allows us to simplify the last inequality to
\[
  z(t) \le 2^{1/\rho} x(0) e^{-bt}.
\]
Plugging into the second differential inequality \eqref{eq:shallow:ode-y} for $y'$, we obtain
\[
  y'
  \le 2^{\frac{1}{2\rho}} c x(0)^{\frac{1}{2}} e^{-\frac{bt}{2}} y^{\frac{1}{2}}.
  =: \bar{c} e^{-\frac{bt}{2}} y^{\frac{1}{2}}.
\]
This inequality is separable and has the solution
\[
  y(t)
  \le \left[ y(0)^{\frac{1}{2}} + \frac{\bar{c}}{b} - \frac{\bar{c}}{b} e^{-\frac{bt}{2}}\right]^2
\]
with upper bound
\[
  y(t) 
  \le \left[ y(0)^{\frac{1}{2}} + \frac{\bar{c}}{b} \right]^2
  = \left[ y(0)^{\frac{1}{2}} + 2^{\frac{1}{2\rho}} \frac{c}{b} x(0)^{\frac{1}{2}} \right]^2 =: \gamma
\]
as stated in \eqref{eq:ode-system-bounds:smoothness} in the lemma. For the derivation of this bound we have assumed that $\bgamma$ is finite, which we can now easily see by contradiction. Indeed, if it is unbounded, set 
\[
  \tau := \inf \{ t<T \, | \, y(t) \ge 2 \gamma \}
\]
to be the smallest time for which $y(t)$ grows beyond $2 \gamma$. For all $t<\tau$ the finiteness assumption is satisfied, so that by our argument above we must have $y(t) \le \gamma$, $t<\tau$. Together with the definition of $\tau$, this contradicts the continuity of $y(t)$ at $t = \tau$.

\end{proof}

\subsection{Interpolation Inequality}

Recall that the $\hs^s$ norms are defined by $\|v\|_s^2 = \sum_{k=0}^\infty \omega_k^{2s} v_k^2$. For completeness, we state the following standard interpolation inequalities, see e.g. \cite{BerghLofstrom1976}.

\begin{lemma} \label{lemma:interpolation-inequality}
  For $s \le r \le t$
  \begin{equation*}
    \|v\|_r \le \|v\|_s^{\frac{t-r}{t-s}} \|v\|_t^{\frac{r-s}{t-s}}.
  \end{equation*}
\end{lemma}

\begin{proof}

Define $p = \frac{t-s}{t-r}$ and $q = \frac{t-s}{r-s}$. The assumption $s \le r \le t$ implies that $p \ge 1$ and $q \ge 1$. Moreover one readily verifies that
\begin{align*}
  \frac{1}{p} + \frac{1}{q} & = 1, &
  \frac{s}{p} + \frac{t}{q} & = r.
\end{align*}
Thus, H\"{o}lder's inequality implies
\begin{multline*}
  \|v\|_r^2
  = \sum_{k=0}^\infty \left[ \omega_k^{2\frac{s}{p}} v_k^{2\frac{1}{p}} \right] \left[\omega_k^{2\frac{t}{q}} v_k^{2\frac{1}{q}} \right]
  \\
  \le \left( \sum_{k=0}^\infty \omega_k^{2s} v_k^2 \right)^{\frac{1}{p}} \left( \sum_{k=0}^\infty \omega_k^{2t} v_k^2 \right)^{\frac{1}{q}}
  =  \|v\|_s^{2\frac{t-r}{t-s}} \|v\|_t^{2\frac{r-s}{t-s}}.
\end{multline*}

\end{proof}

\bibliographystyle{abbrv}
\bibliography{approxopt}

\begin{thebibliography}{10}

\bibitem{AdcockDexter2020}
B.~Adcock and N.~Dexter.
\newblock The gap between theory and practice in function approximation with
  deep neural networks.
\newblock {\em SIAM Journal on Mathematics of Data Science}, 3(2):624–655,
  2021.

\bibitem{Allen-ZhuLiSong2019}
Z.~Allen-Zhu, Y.~Li, and Z.~Song.
\newblock A convergence theory for deep learning via over-parameterization.
\newblock In K.~Chaudhuri and R.~Salakhutdinov, editors, {\em Proceedings of
  the 36th International Conference on Machine Learning}, volume~97 of {\em
  Proceedings of Machine Learning Research}, page 242–252, Long Beach,
  California, USA, 09–15 Jun 2019. PMLR.
\newblock Full version available at \url{https://arxiv.org/abs/1811.03962}.

\bibitem{AroraDuHuEtAl2019}
S.~Arora, S.~Du, W.~Hu, Z.~Li, and R.~Wang.
\newblock Fine-grained analysis of optimization and generalization for
  overparameterized two-layer neural networks.
\newblock In K.~Chaudhuri and R.~Salakhutdinov, editors, {\em Proceedings of
  the 36th International Conference on Machine Learning}, volume~97 of {\em
  Proceedings of Machine Learning Research}, page 322–332, Long Beach,
  California, USA, 09–15 Jun 2019. PMLR.

\bibitem{AroraDuHuEtAl2019a}
S.~Arora, S.~S. Du, W.~Hu, Z.~Li, R.~R. Salakhutdinov, and R.~Wang.
\newblock On exact computation with an infinitely wide neural net.
\newblock In H.~Wallach, H.~Larochelle, A.~Beygelzimer, F.~d'Alché Buc,
  E.~Fox, and R.~Garnett, editors, {\em Advances in Neural Information
  Processing Systems}, volume~32. Curran Associates, Inc., 2019.

\bibitem{Bach2017}
F.~Bach.
\newblock Breaking the curse of dimensionality with convex neural networks.
\newblock {\em Journal of Machine Learning Research}, 18(19):1–53, 2017.

\bibitem{BaiLee2020}
Y.~Bai and J.~D. Lee.
\newblock Beyond linearization: On quadratic and higher-order approximation of
  wide neural networks.
\newblock In {\em International Conference on Learning Representations}, 2020.

\bibitem{BechtelEgert2019}
S.~Bechtel and M.~Egert.
\newblock Interpolation theory for sobolev functions with partially vanishing
  trace on irregular open sets.
\newblock {\em Journal of Fourier Analysis and Applications}, 25(5):2733 –
  2781, 2019.

\bibitem{BerghLofstrom1976}
J.~Bergh and J.~Löfström.
\newblock {\em Interpolation spaces: an introduction}.
\newblock Number 223 in Die {Grundlehren} der mathematischen {Wissenschaften}
  in {Einzeldarstellungen}. Springer, Berlin Heidelberg, 1976.

\bibitem{BernerGrohsKutyniokPetersen2021}
J.~Berner, P.~Grohs, G.~Kutyniok, and P.~Petersen.
\newblock The modern mathematics of deep learning, 2021.
\newblock \url{https://arxiv.org/abs/2105.04026}.

\bibitem{BiettiMairal2019}
A.~Bietti and J.~Mairal.
\newblock On the inductive bias of neural tangent kernels.
\newblock In H.~Wallach, H.~Larochelle, A.~Beygelzimer, F.~d\textquotesingle
  Alché-Buc, E.~Fox, and R.~Garnett, editors, {\em Advances in Neural
  Information Processing Systems}, volume~32. Curran Associates, Inc., 2019.

\bibitem{BreslerNagaraj2020}
G.~Bresler and D.~Nagaraj.
\newblock Sharp representation theorems for {ReLU} networks with precise
  dependence on depth.
\newblock In H.~Larochelle, M.~Ranzato, R.~Hadsell, M.~Balcan, and H.~Lin,
  editors, {\em Advances in Neural Information Processing Systems}, volume~33,
  page 10697–10706. Curran Associates, Inc., 2020.

\bibitem{ChenXu2021}
L.~Chen and S.~Xu.
\newblock Deep neural tangent kernel and laplace kernel have the same rkhs.
\newblock In {\em International Conference on Learning Representations}, 2021.

\bibitem{ChenCaoZouGu2021}
Z.~Chen, Y.~Cao, D.~Zou, and Q.~Gu.
\newblock How much over-parameterization is sufficient to learn deep
  re{\{}lu{\}} networks?
\newblock In {\em International Conference on Learning Representations}, 2021.

\bibitem{ChizatOyallonBach2019}
L.~Chizat, E.~Oyallon, and F.~Bach.
\newblock On lazy training in differentiable programming.
\newblock In H.~Wallach, H.~Larochelle, A.~Beygelzimer, F.~d'Alché Buc,
  E.~Fox, and R.~Garnett, editors, {\em Advances in Neural Information
  Processing Systems}, volume~32. Curran Associates, Inc., 2019.

\bibitem{DaubechiesDeVoreFoucartEtAl2019}
I.~Daubechies, R.~DeVore, S.~Foucart, B.~Hanin, and G.~Petrova.
\newblock Nonlinear {Approximation} and ({Deep}) $\mathrm{ReLU}$ {Networks}.
\newblock {\em Constructive Approximation}, 55(1):127–172, Feb. 2022.

\bibitem{DeVoreHaninPetrova2020}
R.~DeVore, B.~Hanin, and G.~Petrova.
\newblock Neural network approximation.
\newblock {\em Acta Numerica}, 30:327–444, 2021.

\bibitem{DuLeeLiEtAl2019}
S.~Du, J.~Lee, H.~Li, L.~Wang, and X.~Zhai.
\newblock Gradient descent finds global minima of deep neural networks.
\newblock In K.~Chaudhuri and R.~Salakhutdinov, editors, {\em Proceedings of
  the 36th International Conference on Machine Learning}, volume~97 of {\em
  Proceedings of Machine Learning Research}, page 1675–1685, Long Beach,
  California, USA, 09–15 Jun 2019. PMLR.

\bibitem{DuZhaiPoczosSingh2019}
S.~S. Du, X.~Zhai, B.~Poczos, and A.~Singh.
\newblock Gradient descent provably optimizes over-parameterized neural
  networks.
\newblock In {\em International Conference on Learning Representations}, 2019.

\bibitem{ElbraechterPerekrestenkoGrohsEtAl2019}
D.~Elbrächter, D.~Perekrestenko, P.~Grohs, and H.~Bölcskei.
\newblock Deep neural network approximation theory.
\newblock {\em IEEE Transactions on Information Theory}, 67(5):2581–2623,
  2021.

\bibitem{FortDziugaitePaulKharaghaniRoyGanguli2020}
S.~Fort, G.~K. Dziugaite, M.~Paul, S.~Kharaghani, D.~M. Roy, and S.~Ganguli.
\newblock Deep learning versus kernel learning: an empirical study of loss
  landscape geometry and the time evolution of the neural tangent kernel.
\newblock In H.~Larochelle, M.~Ranzato, R.~Hadsell, M.~Balcan, and H.~Lin,
  editors, {\em Advances in Neural Information Processing Systems}, volume~33,
  page 5850–5861. Curran Associates, Inc., 2020.

\bibitem{GeifmanYadavKastenGalunJacobsRonen2020}
A.~Geifman, A.~Yadav, Y.~Kasten, M.~Galun, D.~Jacobs, and B.~Ronen.
\newblock On the similarity between the laplace and neural tangent kernels.
\newblock In H.~Larochelle, M.~Ranzato, R.~Hadsell, M.~Balcan, and H.~Lin,
  editors, {\em Advances in Neural Information Processing Systems}, volume~33,
  page 1451–1461. Curran Associates, Inc., 2020.

\bibitem{GeigerJacotSpiglerGabrielSagundAscoliBiroliHonglerWyart2019}
M.~Geiger, A.~Jacot, S.~Spigler, F.~Gabriel, L.~Sagun, S.~d'Ascoli, G.~Biroli,
  C.~Hongler, and M.~Wyart.
\newblock Scaling description of generalization with number of parameters in
  deep learning.
\newblock {\em CoRR}, abs/1901.01608, 2019.

\bibitem{GentileWelper2022}
R.~Gentile and G.~Welper.
\newblock Approximation rates of a trained neural network, 2022.
\newblock \url{https://github.com/rustygentile/approx-trained}.

\bibitem{GribonvalKutyniokNielsenEtAl2019}
R.~Gribonval, G.~Kutyniok, M.~Nielsen, and F.~Voigtlaender.
\newblock Approximation {Spaces} of {Deep} {Neural} {Networks}.
\newblock {\em Constructive Approximation}, 55(1):259–367, Feb. 2022.

\bibitem{GrohsVoigtlaender2021}
P.~Grohs and F.~Voigtlaender.
\newblock Proof of the theory-to-practice gap in deep learning via sampling
  complexity bounds for neural network approximation spaces, 2021.
\newblock \url{https://arxiv.org/abs/2104.02746}.

\bibitem{GuhringKutyniokPetersen2020}
I.~Gühring, G.~Kutyniok, and P.~Petersen.
\newblock Error bounds for approximations with deep {ReLU} neural networks in
  ws,p norms.
\newblock {\em Analysis and Applications}, 18(05):803–859, 2020.

\bibitem{HaninNica2020}
B.~Hanin and M.~Nica.
\newblock Finite depth and width corrections to the neural tangent kernel.
\newblock In {\em International Conference on Learning Representations}, 2020.

\bibitem{HaoJinSiegelXu2021}
W.~Hao, X.~Jin, J.~W. Siegel, and J.~Xu.
\newblock An efficient greedy training algorithm for neural networks and
  applications in {PDEs}, 2021.
\newblock \url{https://arxiv.org/abs/2107.04466}.

\bibitem{HerrmannOpschoorSchwab2022}
L.~Herrmann, J.~A.~A. Opschoor, and C.~Schwab.
\newblock Constructive deep {ReLU} neural network approximation.
\newblock {\em Journal of Scientific Computing}, 90(2):75, 2022.

\bibitem{HsuKakadeZhang2012}
D.~Hsu, S.~Kakade, and T.~Zhang.
\newblock Tail inequalities for sums of random matrices that depend on the
  intrinsic dimension.
\newblock {\em Electronic Communications in Probability}, 17:1 – 13, 2012.

\bibitem{JacotGabrielHongler2018}
A.~Jacot, F.~Gabriel, and C.~Hongler.
\newblock Neural tangent kernel: Convergence and generalization in neural
  networks.
\newblock In S.~Bengio, H.~Wallach, H.~Larochelle, K.~Grauman, N.~Cesa-Bianchi,
  and R.~Garnett, editors, {\em Advances in Neural Information Processing
  Systems}, volume~31. Curran Associates, Inc., 2018.

\bibitem{JiTelgarsky2020}
Z.~Ji and M.~Telgarsky.
\newblock Polylogarithmic width suffices for gradient descent to achieve
  arbitrarily small test error with shallow {ReLU} networks.
\newblock In {\em International Conference on Learning Representations}, 2020.

\bibitem{JiTelgarskyXian2020}
Z.~Ji, M.~Telgarsky, and R.~Xian.
\newblock Neural tangent kernels, transportation mappings, and universal
  approximation.
\newblock In {\em International Conference on Learning Representations}, 2020.

\bibitem{Kato1973}
T.~Kato.
\newblock Linear evolution equations of “hyperbolic” type, ii.
\newblock {\em Journal of the Mathematical Society of Japan}, 25(4):648 –
  666, 1973.

\bibitem{KawaguchiHuang2019}
K.~Kawaguchi and J.~Huang.
\newblock Gradient descent finds global minima for generalizable deep neural
  networks of practical sizes.
\newblock In {\em 2019 57th Annual Allerton Conference on Communication,
  Control, and Computing (Allerton)}, page 92–99, 2019.

\bibitem{KlusowskiBarron2018}
J.~M. Klusowski and A.~R. Barron.
\newblock Approximation by combinations of {ReLU} and squared {ReLU} ridge
  functions with $\ell^1$ and $\ell^0$ controls.
\newblock {\em IEEE Transactions on Information Theory}, 64(12):7649–7656,
  2018.

\bibitem{KutyniokPetersenRaslanSchneider2022}
G.~Kutyniok, P.~Petersen, M.~Raslan, and R.~Schneider.
\newblock A theoretical analysis of deep neural networks and parametric pdes.
\newblock {\em Constructive Approximation}, 55(1):73–125, 2022.

\bibitem{LaakmannPetersen2020}
F.~Laakmann and P.~Petersen.
\newblock Efficient approximation of solutions of parametric linear transport
  equations by {ReLU} {DNNs}.
\newblock {\em Advances in Computational Mathematics}, 47(1):11, Feb. 2021.

\bibitem{LeeChoiRyuNo2022}
J.~Lee, J.~Y. Choi, E.~K. Ryu, and A.~No.
\newblock Neural tangent kernel analysis of deep narrow neural networks.
\newblock In K.~Chaudhuri, S.~Jegelka, L.~Song, C.~Szepesvari, G.~Niu, and
  S.~Sabato, editors, {\em Proceedings of the 39th International Conference on
  Machine Learning}, volume 162 of {\em Proceedings of Machine Learning
  Research}, page 12282–12351. PMLR, 17–23 Jul 2022.

\bibitem{LeeSchoenholzPenningtonAdlamXiaoNovakSohl-Dickstein2020}
J.~Lee, S.~Schoenholz, J.~Pennington, B.~Adlam, L.~Xiao, R.~Novak, and
  J.~Sohl-Dickstein.
\newblock Finite versus infinite neural networks: an empirical study.
\newblock In H.~Larochelle, M.~Ranzato, R.~Hadsell, M.~Balcan, and H.~Lin,
  editors, {\em Advances in Neural Information Processing Systems}, volume~33,
  page 15156–15172. Curran Associates, Inc., 2020.

\bibitem{LeeXiaoSchoenholzBahriNovakSohlDicksteinPennington2019}
J.~Lee, L.~Xiao, S.~Schoenholz, Y.~Bahri, R.~Novak, J.~Sohl-Dickstein, and
  J.~Pennington.
\newblock Wide neural networks of any depth evolve as linear models under
  gradient descent.
\newblock In H.~Wallach, H.~Larochelle, A.~Beygelzimer, F.~d\textquotesingle
  Alché-Buc, E.~Fox, and R.~Garnett, editors, {\em Advances in Neural
  Information Processing Systems}, volume~32. Curran Associates, Inc., 2019.

\bibitem{LiTangYu2019}
B.~Li, S.~Tang, and H.~Yu.
\newblock Better approximations of high dimensional smooth functions by deep
  neural networks with rectified power units.
\newblock {\em Communications in Computational Physics}, 27(2):379–411, 2019.

\bibitem{LiLiang2018}
Y.~Li and Y.~Liang.
\newblock Learning overparameterized neural networks via stochastic gradient
  descent on structured data.
\newblock In S.~Bengio, H.~Wallach, H.~Larochelle, K.~Grauman, N.~Cesa-Bianchi,
  and R.~Garnett, editors, {\em Advances in Neural Information Processing
  Systems 31}, page 8157–8166. Curran Associates, Inc., 2018.

\bibitem{LiMaWu2020}
Z.~Li, C.~Ma, and L.~Wu.
\newblock Complexity measures for neural networks with general activation
  functions using path-based norms, 2020.
\newblock \url{https://arxiv.org/abs/2009.06132}.

\bibitem{LuShenYangZhang2021}
J.~Lu, Z.~Shen, H.~Yang, and S.~Zhang.
\newblock Deep network approximation for smooth functions.
\newblock {\em SIAM Journal on Mathematical Analysis}, 53(5):5465–5506, 2021.

\bibitem{MarcatiOpschoorPetersenSchwab2022}
C.~Marcati, J.~A.~A. Opschoor, P.~C. Petersen, and C.~Schwab.
\newblock Exponential {ReLU} neural network approximation rates for point and
  edge singularities.
\newblock {\em Foundations of Computational Mathematics}, page 1615 – 3383,
  2022.

\bibitem{Minsker2017}
S.~Minsker.
\newblock On some extensions of bernstein’s inequality for self-adjoint
  operators.
\newblock {\em Statistics \& Probability Letters}, 127:111–119, 2017.

\bibitem{NguyenMondelli2020}
Q.~N. Nguyen and M.~Mondelli.
\newblock Global convergence of deep networks with one wide layer followed by
  pyramidal topology.
\newblock In H.~Larochelle, M.~Ranzato, R.~Hadsell, M.~Balcan, and H.~Lin,
  editors, {\em Advances in Neural Information Processing Systems}, volume~33,
  page 11961–11972. Curran Associates, Inc., 2020.

\bibitem{OpschoorPetersenSchwab2020}
J.~A.~A. Opschoor, P.~C. Petersen, and C.~Schwab.
\newblock Deep {ReLU} networks and high-order finite element methods.
\newblock {\em Analysis and Applications}, 18(05):715–770, 2020.

\bibitem{OymakSoltanolkotabi2020}
S.~{Oymak} and M.~{Soltanolkotabi}.
\newblock Toward moderate overparameterization: Global convergence guarantees
  for training shallow neural networks.
\newblock {\em IEEE Journal on Selected Areas in Information Theory},
  1(1):84–105, 2020.

\bibitem{Pytorch2019}
A.~Paszke, S.~Gross, F.~Massa, A.~Lerer, J.~Bradbury, G.~Chanan, T.~Killeen,
  Z.~Lin, N.~Gimelshein, L.~Antiga, A.~Desmaison, A.~Kopf, E.~Yang, Z.~DeVito,
  M.~Raison, A.~Tejani, S.~Chilamkurthy, B.~Steiner, L.~Fang, J.~Bai, and
  S.~Chintala.
\newblock Pytorch: An imperative style, high-performance deep learning library.
\newblock In H.~Wallach, H.~Larochelle, A.~Beygelzimer, F.~d\textquotesingle
  Alché-Buc, E.~Fox, and R.~Garnett, editors, {\em Advances in Neural
  Information Processing Systems 32}, page 8024–8035. Curran Associates,
  Inc., 2019.

\bibitem{PetersenVoigtlaender2018}
P.~Petersen and F.~Voigtlaender.
\newblock Optimal approximation of piecewise smooth functions using deep {ReLU}
  neural networks.
\newblock {\em Neural Networks}, 108:296–330, 2018.

\bibitem{Pinkus1999}
A.~Pinkus.
\newblock Approximation theory of the mlp model in neural networks.
\newblock {\em Acta Numerica}, 8:143–195, 1999.

\bibitem{PoggioMhaskarRosascoEtAl2017}
T.~Poggio, H.~Mhaskar, L.~Rosasco, B.~Miranda, and Q.~Liao.
\newblock Why and when can deep-but not shallow-networks avoid the curse of
  dimensionality: A review.
\newblock {\em International Journal of Automation and Computing},
  14(5):503–519, 2017.

\bibitem{SeleznovaKutyniok2022a}
M.~Seleznova and G.~Kutyniok.
\newblock Neural tangent kernel beyond the infinite-width limit: Effects of
  depth and initialization.
\newblock In K.~Chaudhuri, S.~Jegelka, L.~Song, C.~Szepesvari, G.~Niu, and
  S.~Sabato, editors, {\em Proceedings of the 39th International Conference on
  Machine Learning}, volume 162 of {\em Proceedings of Machine Learning
  Research}, page 19522–19560. PMLR, 17–23 Jul 2022.

\bibitem{ShahamCloningerCoifman2018}
U.~Shaham, A.~Cloninger, and R.~R. Coifman.
\newblock Provable approximation properties for deep neural networks.
\newblock {\em Applied and Computational Harmonic Analysis}, 44(3):537–557,
  2018.

\bibitem{ShenYangZhang2019}
Z.~Shen, H.~Yang, and S.~Zhang.
\newblock Nonlinear approximation via compositions.
\newblock {\em Neural Networks}, 119:74–84, 2019.

\bibitem{SiegelXu2020}
J.~W. Siegel and J.~Xu.
\newblock Approximation rates for neural networks with general activation
  functions.
\newblock {\em Neural Networks}, 128:313–321, 2020.

\bibitem{SiegelXu2020a}
J.~W. Siegel and J.~Xu.
\newblock High-order approximation rates for shallow neural networks with
  cosine and $\text{ReLU}^k$ activation functions.
\newblock {\em Applied and Computational Harmonic Analysis}, 58:1–26, 2022.

\bibitem{SiegelXu2022}
J.~W. Siegel and J.~Xu.
\newblock Optimal convergence rates for the orthogonal greedy algorithm.
\newblock {\em IEEE Transactions on Information Theory}, 68(5):3354–3361,
  2022.

\bibitem{SongRamezaniKebryaPethickEftekhariCevher2021}
C.~Song, A.~Ramezani-Kebrya, T.~Pethick, A.~Eftekhari, and V.~Cevher.
\newblock Subquadratic overparameterization for shallow neural networks.
\newblock In M.~Ranzato, A.~Beygelzimer, Y.~Dauphin, P.~Liang, and J.~W.
  Vaughan, editors, {\em Advances in Neural Information Processing Systems},
  volume~34, page 11247–11259. Curran Associates, Inc., 2021.

\bibitem{SongYang2019}
Z.~Song and X.~Yang.
\newblock Quadratic suffices for over-parametrization via matrix chernoff
  bound, 2019.
\newblock \url{https://arxiv.org/abs/1906.03593}.

\bibitem{SuYang2019}
L.~Su and P.~Yang.
\newblock On learning over-parameterized neural networks: A functional
  approximation perspective.
\newblock In H.~Wallach, H.~Larochelle, A.~Beygelzimer, F.~d\textquotesingle
  Alché-Buc, E.~Fox, and R.~Garnett, editors, {\em Advances in Neural
  Information Processing Systems}, volume~32. Curran Associates, Inc., 2019.

\bibitem{Suzuki2019}
T.~Suzuki.
\newblock Adaptivity of deep re{LU} network for learning in besov and mixed
  smooth besov spaces: optimal rate and curse of dimensionality.
\newblock In {\em International Conference on Learning Representations}, 2019.

\bibitem{Tropp2015}
J.~A. Tropp.
\newblock An introduction to matrix concentration inequalities.
\newblock {\em Foundations and Trends® in Machine Learning}, 8(1-2):1–230,
  2015.

\bibitem{VelikanovYarotsky2021}
M.~Velikanov and D.~Yarotsky.
\newblock Universal scaling laws in the gradient descent training of neural
  networks, 2021.
\newblock \url{https://arxiv.org/abs/2105.00507}.

\bibitem{VelikanovYarotsky2022}
M.~Velikanov and D.~Yarotsky.
\newblock Tight convergence rate bounds for optimization under power law
  spectral conditions, 2022.
\newblock \url{https://arxiv.org/abs/2202.00992}.

\bibitem{Vershynin2018}
R.~Vershynin.
\newblock {\em High-dimensional probability: an introduction with applications
  in data science}.
\newblock Number~47 in Cambridge series in statistical and probabilistic
  mathematics. Cambridge University Press, Cambridge ; New York, NY, 2018.

\bibitem{WeinanChaoLeiWojtowytsch2020}
E.~Weinan, M.~Chao, W.~Lei, and S.~Wojtowytsch.
\newblock Towards a mathematical understanding of neural network-based machine
  learning: What we know and what we don't.
\newblock {\em CSIAM Transactions on Applied Mathematics}, 1(4):561–615,
  2020.

\bibitem{WeinanMaWu2019}
E.~Weinan, C.~Ma, and L.~Wu.
\newblock The {Barron} {Space} and the {Flow}-{Induced} {Function} {Spaces} for
  {Neural} {Network} {Models}.
\newblock {\em Constructive Approximation}, 55(1):369–406, Feb. 2022.

\bibitem{Yarotsky2017}
D.~Yarotsky.
\newblock Error bounds for approximations with deep {ReLU} networks.
\newblock {\em Neural Networks}, 94:103–114, 2017.

\bibitem{Yarotsky2018}
D.~Yarotsky.
\newblock Optimal approximation of continuous functions by very deep {ReLU}
  networks.
\newblock In S.~Bubeck, V.~Perchet, and P.~Rigollet, editors, {\em Proceedings
  of the 31st Conference On Learning Theory}, volume~75 of {\em Proceedings of
  Machine Learning Research}, page 639–649. PMLR, 06–09 Jul 2018.

\bibitem{YarotskyZhevnerchuk2020}
D.~Yarotsky and A.~Zhevnerchuk.
\newblock The phase diagram of approximation rates for deep neural networks.
\newblock In H.~Larochelle, M.~Ranzato, R.~Hadsell, M.~Balcan, and H.~Lin,
  editors, {\em Advances in Neural Information Processing Systems}, volume~33,
  page 13005–13015. Curran Associates, Inc., 2020.

\bibitem{ZouCaoZhouGu2020}
D.~Zou, Y.~Cao, D.~Zhou, and Q.~Gu.
\newblock Gradient descent optimizes over-parameterized deep {ReLU} networks.
\newblock {\em Machine Learning}, 109(3):467 – 492, 2020.

\bibitem{ZouGu2019}
D.~Zou and Q.~Gu.
\newblock An improved analysis of training over-parameterized deep neural
  networks.
\newblock In H.~Wallach, H.~Larochelle, A.~Beygelzimer, F.~d\textquotesingle
  Alché-Buc, E.~Fox, and R.~Garnett, editors, {\em Advances in Neural
  Information Processing Systems}, volume~32. Curran Associates, Inc., 2019.

\end{thebibliography}

\end{document}